\documentclass{article}
\usepackage[margin=1in]{geometry}
\usepackage[utf8]{inputenc}
\usepackage[numbers]{natbib}
\usepackage{authblk, caption, floatrow}
\usepackage{palatino}
\usepackage{amsmath,amsthm,amsfonts,amssymb}
\usepackage{bm, bbm}
\usepackage[pdftex]{graphicx}
\usepackage{xcolor,xparse}
\usepackage{subcaption}
\usepackage{multirow}
\usepackage{makecell}
\usepackage{algorithm}
\usepackage{algorithmic}
\usepackage{hyperref}
\usepackage{mathtools}
\usepackage{ifthen}
\mathtoolsset{showonlyrefs=false}

\newtheorem{theorem}{Theorem}
\newtheorem{proposition}[theorem]{Proposition}

\renewcommand{\algorithmiccomment}[1]{\hfill $\triangleright$ #1}

\newcommand{\SYdens}{p}
\newcommand{\correctdist}{P_1}
\newcommand{\correctdens}{p_1}
\newcommand{\correctdenst}{p_1}

\newcommand{\wrongdist}{P_0}
\newcommand{\wrongdens}{p_0}
\newcommand{\wrongdenst}{p_0}

\newcommand{\incorrectdens}{\wrongdens}

\newcommand{\nullh}{\mathcal{H}_N}

\newcommand{\Dcal}{\mathcal{D}_\text{cal}}

\title{\fontsize{14.5}{16}\selectfont E-valuator: Reliable Agent Verifiers with Sequential Hypothesis Testing}
\author{
Shuvom Sadhuka\textsuperscript{1,2},
Drew Prinster\textsuperscript{1,3},
Clara Fannjiang\textsuperscript{1},
Gabriele Scalia\textsuperscript{1},
Bonnie Berger\textsuperscript{2},
Aviv Regev\textsuperscript{1},
Hanchen Wang\textsuperscript{1,4}
\\
\textsuperscript{1}Genentech \quad
\textsuperscript{2}MIT \quad
\textsuperscript{3}Johns Hopkins \quad
\textsuperscript{4}Stanford
}

\begin{document}

\maketitle

\begin{abstract}
\noindent\textit{Agentic} AI systems execute a sequence of actions, such as reasoning, coding or tool calls, in response to user prompts.
To evaluate the success of their trajectories, researchers have developed verifiers, such as process-reward models, to score the quality of each action in an agent's trajectory. Although these heuristic scores can be informative, there are no guarantees of correctness when used to decide whether an agent will yield a successful output.
Here, we introduce \textit{e-valuator}, a method to convert any black-box verifier score into a decision rule with provable control of false alarm rates.
We frame the problem of distinguishing a successful trajectory---that is, a sequence of actions that will lead to a correct response to the user's prompt---from an unsuccessful trajectory as a sequential hypothesis testing problem.
\textit{E-valuator} develops a sequential hypothesis test that remains valid at every step of an agent's trajectory, enabling online monitoring of agents.
Empirically, we demonstrate that \textit{e-valuator} provides greater statistical power and better false alarm rate control than other strategies across six datasets and three agents.
We additionally show that \textit{e-valuator} can quickly terminate unsuccessful trajectories to save tokens. Together, \textit{e-valuator} provides a lightweight framework that converts verifier heuristics into decision rules with statistical guarantees, enabling the deployment of reliable agents.

\end{abstract}

\section{Introduction}

\textit{Agents} are black-box systems that autonomously perform tasks by executing a sequence of actions, called a trajectory.
The actions may include interacting with an external environment through tool calls, writing code, or steps of logical reasoning.
In this work, we use \textit{agents} to specifically denote large language model (LLM)-based systems that respond to user requests; more broadly, the \textit{agent} can refer to robotic agents that execute physical tasks through sequences of mechanical actions \citep{bekey1998autonomous} or game-playing agents that play games through sequences of game-legal actions \citep{brown2019superhuman, silver2017mastering}.
Agents have wide applications, and have delivered promising results across diverse domains, including identifying drug candidates to target diseases \citep{gottweis2025towards, swanson2025virtual}, mining spatial transcriptomic datasets \citep{wang2025spatialagent, huang2025biomni}, and verifying scientific hypotheses  \citep{huang2025automated}.

Nonetheless, agents make mistakes, and it is important to be able to detect these \citep{rabanser2026towards}. To this end, \textit{verifier} models have been developed to numerically score each action in an agent's trajectory. 
These scores are typically used as a proxy for the probability that the trajectory will successfully produce a correct final output. 
Example verifiers include judge LLMs, which provide a score (as text output) after each step \citep{li2025generation}, and process-reward models, which are finetuned to give a prediction of whether each step in a trajectory is ``correct" or ``incorrect" \citep{li2025process, zheng2025processbench, lightman2024let}. 
These verifiers' scores can then be used to identify unsuccessful trajectories, as those that will produce an incorrect final output.

A key limitation of verifiers to date is that when their scores are used to decide whether to flag a trajectory as incorrect, there are no guarantees on the resulting probability of error of this downstream decision.
Rigorous guarantees may become particularly critical when agents are deployed in high-stakes settings with real-world implications, such as autonomous labs \citep{leong2025steering}, gene editing \citep{qu2026crispr}, or hospital operations \citep{Gebreab2024-ej}.
In particular, we focus on guaranteeing control over the \textit{false alarm rate}, or the probability of incorrectly flagging a successful trajectory as unsuccessful, as a way of ensuring that any alarms are statistically trustworthy.

Even if verifier scores satisfy popular notions of probabilistic ``correctness,'' such as marginal calibration \citep{guo2017calibration}, such notions do not provide guarantees on the false alarm rate. Furthermore, each action in a trajectory costs time and resources, and trajectories can be long. We therefore want to detect that a trajectory will be unsuccessful early on---ideally, after as few actions as possible---instead of having to incur the costs of rolling out the full trajectory before making the decision. These desiderata necessitate continuously assessing the verifier score (e.g., after each action), a sequential testing problem that exacerbates the false alarm issues of naive approaches, particularly since the length of the complete trajectory is not known in advance.

Fine-tuning or building better verifiers does not directly provide error guarantees. Furthermore, fine-tuning verifiers requires sufficient compute as well as white-box access to the verifier (and possibly the agent) weights.
Even with sufficient compute and access to the verifier weights, it may be impractical to obtain sufficient data appropriate for fine-tuning verifiers: one needs a ``correctness'' label for every action in every trajectory \citep{lightman2024let}.

\begin{figure*}[t]
    \vspace{-0.5em}
    \centering
    \includegraphics[width=\textwidth, trim=0 200 0 150, 
    clip]{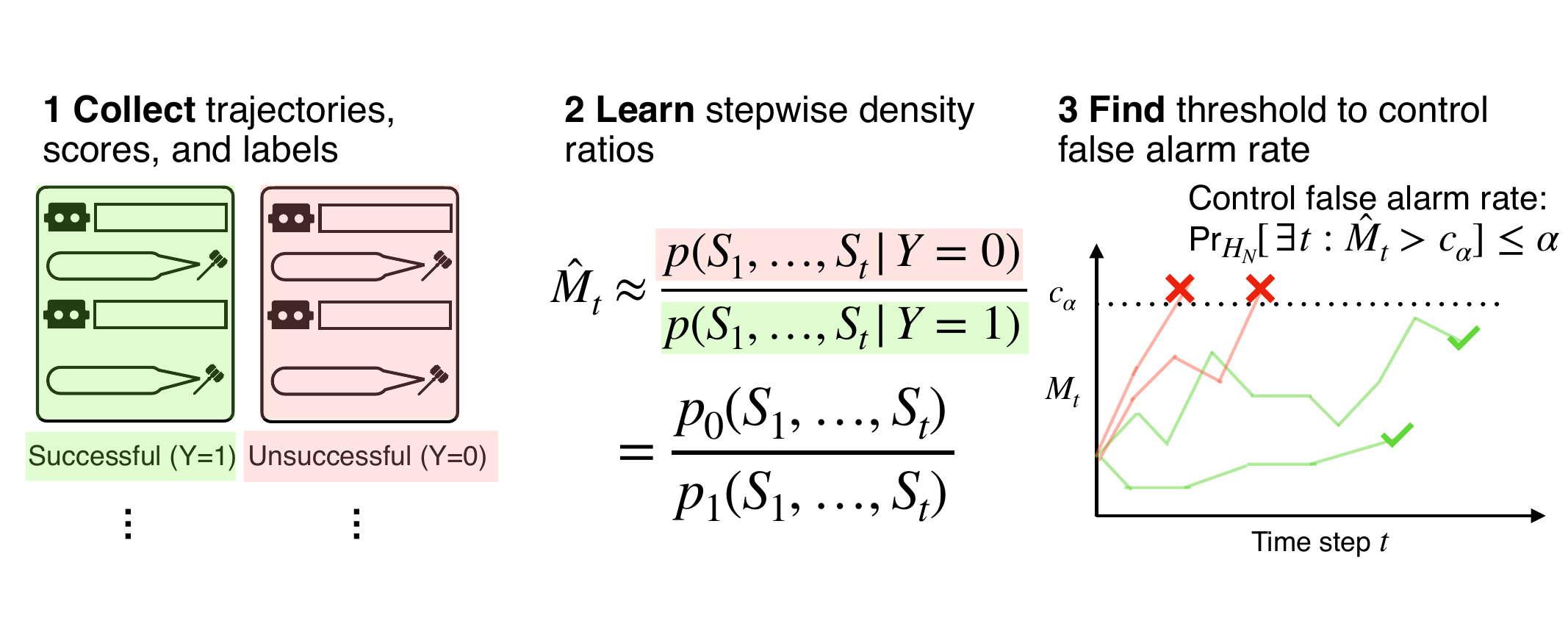}
    \caption{\textbf{E-valuator overview}. \textit{E-valuator} works in three steps. First, we collect a small calibration set of trajectories, verifier scores, and labels. Second, we learn the density ratios $\hat{M}_t$ at each timestep $t$.
    Third, we find a decision threshold that controls the false alarm rate with high probability. 
    Using this threshold, unsuccessful trajectories (red) are rejected at a higher rate than successful ones (green). 
    }
    \label{fig:overview}
    \vspace{-2em} 
\end{figure*}

To tackle these challenges, we introduce \textit{e-valuator}, a lightweight statistical wrapper that converts scores from any black-box verifier into a decision rule for detecting unsuccessful agent trajectories, with guarantees on the false alarm rate (see Figure ~\ref{fig:overview}). To do so, we first frame the problem as a hypothesis test. 
We assume that the sequence of verifier scores is drawn from one distribution, $\correctdens$, for successful trajectories (i.e., those that will produce a correct final output), and another distribution, $\incorrectdens$, for  unsuccessful trajectories.
Given a new trajectory, the problem then reduces to deciding, after as few actions as possible, whether the verifier score sequence is drawn from $\correctdens$ or $\incorrectdens$. 

In designing this hypothesis test, we make two contributions. 
First, we provide a probably-approximately-correct (PAC) thresholding procedure that controls the false alarm rate with high probability for \textit{any} sequential scoring process, using a small calibration set. 
Second, we draw on prior work in sequential hypothesis testing and e-processes \citep{ramdas2024hypothesis, wald1948optimum} to construct a specific sequential process atop the black-box verifier scores---the density ratios between $p_1$ and $p_0$ for the partial trajectories at each step---that is \textit{log-optimal}, meaning it accumulates evidence against unsuccessful trajectories as fast as possible. In practice, $p_1$ and $p_0$ are not known, so we learn a model of this ratio from the calibration set and account for potential estimation error via our PAC thresholding procedure.

The density ratio construction in our second contribution draws on e-values and sequential hypothesis testing. E-values are an alternative to $p$-values that are particularly useful in settings where one wants to run a sequence of hypothesis tests (e.g., ``is this ongoing trajectory successful?'') but might not know the number of tests beforehand or may want a cumulative metric of evidence over time. In our setting, for instance, trajectories are of variable length, and a practitioner may wish to either terminate unpromising trajectories as soon as possible or continually track a metric of reliability.

Importantly, \textit{e-valuator} complements any future improvements to verifiers: while it guarantees control of the false alarm rate for any agent and verifier (with respect to a calibration set), deploying it with better verifiers will tend to yield decision rules with higher power---that is, higher rates of correctly flagging unsuccessful trajectories. \textit{E-valuator} requires minimal compute and can run on a standard laptop.

Empirically, across six datasets and three agents, \textit{e-valuator} provided better false alarm rate control and statistical power than other baselines, such as a raw, calibrated, or PAC thresholded verifier alone. We also show that e-valuator terminates unsuccessful trajectories earlier than these baselines, which enables more favorable tradeoffs between token usage and accuracy. For instance, e-valuator recovers up to 90\% of an agent's original accuracy with just 80\% of the tokens on one dataset tested.

\section{Methods}
\subsection{Problem Setting}\label{sec:overview}
Given a user prompt, denoted $o_0$, the agent executes a sequence of $T \in \mathbb{N}_+$ actions, $(a_1, \ldots, a_T)$, where $\mathbb{N}_+ := \{1, 2, ...\}$ and $T$ is a random variable depending on both the prompt and the agent's internal randomness.
Associated with each action, $a_t$, is an observation, $o_t$, which captures the environment state after the action is executed (e.g., if $a_t$ performs some intermediary arithmetic calculation, $o_t$ could contain the calculated value).
The actions and observations together form the \textit{trajectory} at each step $t$, $H_t = (o_0, a_1, \ldots, a_t, o_t)$.
After each step $t$, a black-box verifier model $v$ takes as input the trajectory $H_t$ and provides a score, $S_t = v(H_t)$, which serves as a heuristic evaluation of the quality of the trajectory thus far.
Typically, $S_t \in [0,1]$, although \textit{e-valuator} supports score values of any type.
The verifier scores form a sequence $\mathbf{S} = (S_1, \ldots, S_T)$, and we denote $\mathbf{S}_{[1:t]} = (S_1, \ldots, S_t)$. A complete trajectory $H_T$ is associated with a binary label, $Y \in \{0,1\}$, of whether the final output, $o_T$, is correct $(Y = 1)$ or not $(Y = 0)$. We call trajectories where $Y=1$ \textit{successful} trajectories and those with $Y=0$ \textit{unsuccessful} trajectories.


\subsection{Evaluation via hypothesis testing}
\textbf{Sequential hypothesis test.} Given a new score sequence, $\mathbf{S}$, our goal is to determine whether the agent's trajectory will produce a correct ($Y=1$) or incorrect ($Y=0$) final output.
We formalize this goal as a hypothesis test.
Let $\correctdist$ and $\wrongdist$ denote the distributions of score sequences conditioned on correct and incorrect final outputs, respectively, with corresponding densities $\correctdens$ and $\wrongdens$.
That is, $\correctdens(\mathbf{S}) = \SYdens(\mathbf{S} \mid Y = 1)$ and $\wrongdens(\mathbf{S}) = \SYdens(\mathbf{S} \mid Y = 0)$.

 Assuming $P_1  \neq P_0$\footnote{If $P_1 = P_0$, the score sequences would be identically distributed regardless of whether the trajectory is successful or not, meaning the verifier scores carry no information about correctness. In this case, no test based on the scores can do better than random guessing.}, we test between a ``null'' hypothesis, $\mathcal{H}_N$, and an ``alternative'' hypothesis, $\mathcal{H}_A$:
\begin{align*}
    \mathcal{H}_N &: \mathbf{S} \sim \correctdist \;\;\; \text{(the final output is correct)}\\ 
    \mathcal{H}_A&: \mathbf{S} \sim \wrongdist \;\;\; \text{(the final output is incorrect).}
\end{align*}
Note that $\correctdist$ and $\wrongdist$ generally encode complex dependencies between the scores over time.
That is, black-box verifier scores are generally not independent samples from a fixed distribution at every step, or otherwise amenable to convenient assumptions.

We construct a \textit{sequential test} between $\mathcal{H}_N$ and $\mathcal{H}_A$ that can be run at each step $t$, using only $\mathbf{S}_{[1:t]}$, the scores based only on the agent's trajectory up to step $t$.
Specifically, we construct a sequence of test statistics, $(M_t)_{t=1}^T$ where $M_t$ is the output of a real-valued function of $\mathbf{S}_{[1:t]}$, and a real-valued decision threshold, $c_\alpha$, given a user-specified error level, $\alpha \in (0, 1)$.

\textbf{False alarm rate control.} For $t = 1, 2, \ldots$, if $M_t > c_\alpha$, we reject $\mathcal{H}_N$.
If we reach $t = T$, the end of the agent's trajectory, without rejecting $\mathcal{H}_N$, then we accept $\mathcal{H}_N$. Our primary goal in designing this sequence $(M_t)_{t=1}^T$ and selecting the threshold $c_\alpha$ is to control the \textbf{false alarm rate}. We define \textbf{false alarm rate (FAR) control} as\footnote{One could instead control the missed detection rate, i.e., the probability of failing to flag an unsuccessful trajectory as unsuccessful. In this case, one would swap $\mathcal{H}_N$ and $\mathcal{H}_A$ and apply the same techniques. We leave the choice of which error rate to control to the user.}
\begin{equation}\label{eq:far}
    \text{Pr}_{\mathcal{H}_N}[\exists \ t \in [T] : M_t > c_\alpha] \leq \alpha
\end{equation}
where $T$, the length of the complete trajectory, is also random, and $[T] := \{1, \ldots, T\}$. In words, the \textit{false alarm rate}, or the rate at which we reject successful trajectories, is guaranteed to be at most $\alpha$. Notice that the probability under consideration is that $M_t$ \textit{ever} surpasses $c_\alpha$---equivalently, that we \textit{ever} reject the null hypothesis, $\mathcal{H}_N$---regardless of $T$, the total number of actions the agent takes (which varies from trajectory to trajectory).

Of course, controlling the false alarm rate alone is insufficient; the trivial decision rule that never rejects has a false alarm rate of zero.
Indeed, the choice of $(M_t)_{t=1}^T$ determines how quickly the test accumulates evidence against the null ($\mathcal{H}_N$) when the trajectory is unsuccessful ($\mathcal{H}_A$). We will return to the choice of $(M_t)_{t=1}^T$ in Section \ref{sec:logoptimal}. 

Concretely, our method randomly splits our calibration data, $\mathcal{D}_\text{cal}$, into disjoint sets, $\mathcal{D}_\text{DRE}$ and $\mathcal{D}_\text{threshold}$, such that $\mathcal{D}_\text{cal} =  \mathcal{D}_\text{DRE} \cup\mathcal{D}_\text{threshold}$.
We design the scoring process, $(M_t)_{t=1}^T$, based on $\mathcal{D}_\text{DRE}$ (Section~\ref{sec:logoptimal}), and select the decision threshold, $c_\alpha$, using $\mathcal{D}_\text{threshold}$ (Section~\ref{sec:pac}).  

\subsection{False alarm rate control via PAC thresholding}\label{sec:pac}
\begin{algorithm}[tb]
\caption{Probably-approximately-correct (PAC) threshold.}
\label{alg:quantile}
\begin{algorithmic}[1]
   \STATE \textbf{Inputs:} error level $\delta \in [0, 1]$; quantile level $\alpha \in [0, 1]$; calibration data $\mathcal{D}_\text{threshold}$; functions that score the process at each step $t$, $\{f_t\}$.
   \STATE \textbf{Output:} decision threshold $c_\alpha$ that achieves FAR control with high probability.
   \STATE $n \leftarrow |\{(\mathbf{S}, Y) \in \mathcal{D}_\text{threshold} : Y = 1\}|$
   \IF{$(1-\alpha)^n > \delta$}
       \STATE $c_\alpha \leftarrow \infty$ \COMMENT{Too few calibration samples for given $\alpha, \delta$}
       \STATE \textbf{return} $c_\alpha$
       
   \ENDIF
   \FOR{$(\mathbf{S}^{(i)}, Y^{(i)}) \in \mathcal{D}_\text{threshold}$ such that $Y^{(i)} = 1$}
       \STATE $M_1, \ldots, M_T \leftarrow f_1(\mathbf{S}_{[1:1]}^{(i)}), \ldots, f_T(\mathbf{S}_{[1:T]}^{(i)})$
       \STATE $M^{(i)} \leftarrow \max_t M_t$ \COMMENT{Maximum for $i$-th null sample}
   \ENDFOR
   \STATE Sort $M^{(1)}, \ldots, M^{(n)}$ in ascending order, $M_{(1)} \leq \cdots \leq M_{(n)}$. \COMMENT{Break ties by flipping fair coin}
   \STATE $k \leftarrow \min\{j \in [n]: \Pr[\text{Bin}(n, 1 - \alpha) \geq j] \leq \delta\}$ \COMMENT{Find index using Binomial tail bound}
   \STATE $c_\alpha \leftarrow M_{(k)}$
\end{algorithmic}
\end{algorithm}
In this section, we provide a procedure that controls the FAR (Eq. \ref{eq:far}) for \textit{any} process $(M_t)_{t=1}^T$. 
In particular, let $(f_t)_{t=1}^T$ be a sequence of functions where $f_t: \mathbb{R}^t \rightarrow \mathbb{R}$ takes as input the (partial) sequence $\mathbf{S}_{[1:t]} \in \mathbb{R}^t$ and returns $M_t \in \mathbb{R}$ (e.g., $f_t$ could be the identity function, $M_t = f_t(\mathbf{S}_{[1:t]}) = S_t$). We will return to the design of $(f_t)_{t=1}^T$ in Section \ref{sec:logoptimal}.

Our goal is to find a threshold $c_\alpha$ for any $\alpha \in (0,1)$, such that $\text{Pr}_{\mathcal{H}_N}[\exists \ t \in [T] : M_t > c_\alpha] \leq \alpha$.  Note that rejecting $\nullh$ if the process $M_t$ ever surpasses $c_\alpha$ is equivalent to rejecting $\nullh$ if $\max_t M_t > c_\alpha$. Thus, it suffices to set $c_\alpha$ to the $(1-\alpha)$ quantile of the distribution of $\max_t M_t$ under the null.

Accordingly, to set $c_\alpha$, we focus on successful trajectories corresponding to $\nullh$. For each $i$th trajectory in the threshold calibration set $\mathcal{D}_\text{threshold}$, we record the maximum $M_t$ over all steps, $M^{(i)} = \text{max}\{M_1^{(i)}, \ldots, M_T^{(i)}\}$. This maximum value, $M^{(i)}$, is a sample from the null distribution.
Given these samples, we construct $\hat{q}_{1-\alpha}$, a high-probability upper bound on the $(1 - \alpha)$-quantile of the null distribution (Alg.~\ref{alg:quantile}). We refer to this procedure as \textbf{probably-approximately-correct (PAC)} thresholding. In particular, with probability at least $1 - \delta$ for any user-specified $\delta$ (``probably''), the procedure controls the false alarm rate under a user-specified $\alpha$ (``approximately correct'').
\begin{proposition}[PAC thresholding]
    \label{prop:pac}
    Let $\{M_t\}_{t \in \mathbb{N}_+}$ denote any scoring process where $M_t = f_t(\mathbf{S}_{[1:t]})$ for a sequence of (deterministic) functions $(f_t)_{t \in \mathbb{N}_+}$ and $M_t \in \mathbb{R}$.
    For fixed error level $\delta \in (0, 1)$ and quantile level $\alpha \in (0, 1)$, let $c_\alpha$ be the output of Algorithm~\ref{alg:quantile}.
    Then,
    \begin{align*}  \text{Pr}_{\mathcal{D}_{cal}} \left(\text{Pr}_{\nullh} \left(\exists \ t  \in [T]: M_t > c_\alpha \mid \mathcal{D}_\text{cal} \right) \leq \alpha \right) \geq 1 - \delta.
    \end{align*}
\end{proposition}

See Appendix~\ref{sec:proof_pac} for the proof. The thresholding step in Algorithm~\ref{alg:quantile} is similar to conformal procedures \citep{vovk2005algorithmic, bates2023testing}: a distribution-free quantile estimated from order statistics of a calibration set, with the index chosen via a binomial tail bound. We apply it to the maximum of a sequentially-computed scoring process $(M_t)_{t=1}^T$, which is what yields FAR control over a random-length trajectory.

Finally, observe that Proposition~\ref{prop:pac} involves two error parameters: $\alpha$ (the quantile level) and $\delta$ (the probability of FAR miscalibration). By a union bound, the marginal false alarm rate satisfies $\Pr_{\nullh}[\exists\, t \in [T] : M_t > c_\alpha] \leq \alpha + \delta$. Thus, to control the marginal false alarm rate at a desired level $\alpha_0$, one can choose any $\alpha$ and $\delta$ such that $\alpha + \delta \leq \alpha_0$. Additionally, the choice of $\alpha$ and $\delta$ is constrained by the number of samples with $Y=1$ the calibration set, denoted $n$ (see L3-6 of \ref{alg:quantile}). If $n$ is too small, the algorithm returns $c_\alpha = \infty$ and the procedure never rejects (trivially controlling FAR but obtaining zero power). This occurs when $n < \log \delta / \log(1-\alpha)$ (see Appendix~\ref{sec:proof_pac} for details). Thus, we recommend setting $\alpha, \delta$ such that $n \geq \lceil \log \delta / \log (1-\alpha) \rceil$.

\subsection{Designing a (log-)optimal test statistic}\label{sec:logoptimal}
Although Proposition \ref{prop:pac} controls the false alarm rate with high probability for any function $(f_t)_{t \in \mathbb{N}_+}$, it does not guide \textit{which} functions we should use. 
Our test rejects a trajectory when $M_t$ first exceeds $c_\alpha$. Thus, an ideal $(f_t)_{t \in \mathbb{N}_+}$ will have $M_t$ grow rapidly under $\mathcal{H}_A$ (when the trajectory is unsuccessful).

\textbf{E-processes.} To do so, we draw from \textit{e-processes} \citep{ramdas2024hypothesis}, a framework for sequential hypothesis testing that enables principled construction of $(f_t)_{t \in \mathbb{N}}$ and additional theoretical guarantees. An e-process for $\mathcal{H}_N$ is a non-negative stochastic process, $(E_t)_t$, such that each $E_t$ is an \textit{e-value} for $\mathcal{H}_N$---that is, $\mathbb{E}_{\mathcal{H}_N}[E_t] \leq 1$---and there exists a \textit{test martingale} for $\mathcal{H}_N$, $(M_t)_t$, such that $E_t \leq M_t$ always.
A test martingale for $\mathcal{H}_N$ is a sequence, $(M_t)_t$, that satisfies: (a) \textbf{Non-negativity and unit mean}: $M_t$ is non-negative for all $t$ and $E_{\mathcal{H}_N}[M_0] \leq 1$. (b) \textbf{Martingale}. $(M_t)_t$ is a martingale for $\mathcal{H}_N$. For all $t$, $\mathbb{E}_{\mathcal{H}_N}[M_t \mid M_0, \ldots, M_{t-1}] = M_{t-1}$. Note that any test martingale is an e-process, but that e-processes are a broader class of processes (not just test martingales). 

E-processes are particularly useful as their concentration behavior can be exploited to provide direct FAR control. In particular, when $(M_t)_{t=1}^\infty$ is an exact e-process, Ville's inequality \citep{doob1939jean} gives a distribution-free and time-uniform threshold: $\text{Pr}_{\mathcal{H}_N}[\exists \ t \in \mathbb{N}_+ : M_t \geq 1/\alpha] \leq \alpha$. That is, setting $c_\alpha = 1/\alpha$ achieves FAR control without any calibration data when $(M_t)_{t=1}^\infty$ is an exact e-process.

\textbf{Density ratio is log-optimal.} The density ratio process provides a natural test martingale for our setting. In particular, we use the ratio between the alternative density, $p_0$, and null density, $p_1$ at each step $t$. Specifically, set $M_0 = 1$, and, for each step $t \in [T]$,
\begin{align}\label{eq:dens_ratio}
    M_t & = f_t(\mathbf{S}_{[1:t]}) = \frac{\wrongdenst(\mathbf{S}_{[1:t]})}{\correctdenst(\mathbf{S}_{[1:t]})}.
\end{align}
It turns out that under $\mathcal{H}_A$, the density ratio process is log-optimal, meaning it grows the fastest (in expectation, log-scaled) over time among all e-processes. That is, if $M_t = \frac{\wrongdenst(\mathbf{S}_{[1:t]})}{\correctdenst(\mathbf{S}_{[1:t]})}$ denotes the density ratio process, then for any other e-process $(M_t')_{t=0}^T$ and stopping time $\tau$, $E_{\mathcal{H}_A}[\log M_\tau] \geq E_{\mathcal{H}_A} [\log M_\tau']$. See Appendix~\ref{sec:appdx:logopt} for a formal statement and proof of this. Log-optimality is analogous to being the ``most powerful" test statistic in non-sequential hypothesis testing: intuitively, $M_t$ will tend to surpass the decision threshold, and correspondingly enable detection of unsuccessful trajectories earlier than other e-processes. 


In practice, the forms of $\correctdenst$ and $\wrongdenst$ are typically  not known.
Prior to deployment, the method we propose, \textit{e-valuator}, therefore has a calibration phase in which it uses a split of the calibration data, $\mathcal{D}_\text{DRE}$, to learn a model of the density ratio, $\hat{M}_t(\mathbf{S}_{[1:t]}) \approx \frac{\wrongdenst(\mathbf{S}_{[1:t]})}{\correctdenst(\mathbf{S}_{[1:t]})}$, for each step $t$. Indeed, even if one cannot exactly evaluate the true densities, one can learn the density ratios and approximate the log-optimal tests, as is done in prior works \citep{dandapanthula2025offline}. We expand on the density ratio estimation below.


\subsection{Density ratio estimation}\label{sec:dre}
To estimate our density ratios, we use classifier-based density ratio estimation. In particular, observe the following equality by Bayes' rule \citep{bickel2009, gutmann2012noise}:
\begin{align}
\label{eq:bayes}
    M_t = f_t(\mathbf{S}_{[1:t]}) = \frac{\wrongdenst(\mathbf{S}_{[1:t]})}{\correctdenst(\mathbf{S}_{[1:t]})} = \frac{p(Y=0|\mathbf{S}_{[1:t]})}{p(Y=1|\mathbf{S}_{[1:t]})} \cdot \frac{p(Y = 1)}{p(Y = 0)}.
\end{align}
Thus, for each time step $t$, we train a classifier, $\hat{g}_t$, which takes $\mathbf{S}_{[1:t]}$ as input and provides an estimate of $p (Y=1|\mathbf{S}_{[1:t]})$.
We also form an estimate, $\hat{\pi}_1$, of the class probability $p(Y = 1)$ from the calibration set.
We plug these two estimates into Eq.~\eqref{eq:bayes} to form the following estimated density ratio at step $t$:
\begin{align*}
    M_t &= \hat{f}_t(\mathbf{S}_{[1:t]}) = \underbrace{\frac{1 - \hat{g}_t(\mathbf{S}_{[1:t]})}{\hat{g}_t(\mathbf{S}_{[1:t]})}}_{\text{classifier-based estimate}} \cdot \underbrace{\frac{\hat{\pi}_{1}}{1 - \hat{\pi}_{1}}}_{\text{prior odds estimate}}
\end{align*}
We then select the threshold $c_\alpha$ using the procedure in Proposition \ref{prop:pac}, which applies to \textit{any} scoring process with i.i.d. trajectories from the null available in the calibration set. In our experiments, we used simple logistic regression for the classifier at each step, $\hat{g}_t$, and found that estimated density ratios learned from a few hundred calibration points empirically achieved both false alarm rate control and superior power to alternative methods (see Appendix \ref{sec:cal_size}). 

\section{Related Work}

\textbf{Verifiers and PRMs.} As \textit{e-valuator} is a statistical wrapper for any verifier, our work is relevant to prior work on building better verifiers. These verifiers are often trained as models that estimate a reward (e.g., correctness) after each step in an agent's action sequence. Among these, process-reward models (PRMs) are finetuned using agent trajectories where each \textit{step} is labeled as correct or not \citep{uesato2022solving, wang2024math, khalifa2025process}. Some prior works also calibrate existing PRMs \citep{you2025probabilistic, park2026know}, although calibration alone is insufficient to control false alarm rates. Training PRMs can be expensive, as it requires (a) access to human-annotated process labels \citep{lightman2024let} and (b) finetuning existing LLMs \citep{wang2024math}. Alternative verifiers include LLM-as-a-judge \citep{bavaresco2024llms} and outcome reward models, which only provide a label for the entire trajectory \citep{creswellselection}. There are also several benchmarks to compare verifiers \citep{lu2025agentrewardbench,zheng2025processbench}.

\textbf{Sequential hypothesis testing and e-values.} E-valuator directly builds on prior work in e-values, which have useful properties for sequential hypothesis testing \citep{ramdas2024hypothesis, vovk2021values, vovk2023confidence}. E-values provide anytime validity over a (potentially infinite) sequence of tests \citep{wang2025anytime, waudby2024anytime, waudby2023distribution, grunwald2020safe}. They have found applications in many settings, including A/B testing \citep{johari2017peeking}, changepoint detection \citep{shin2023detectors, lorden1971procedures}, and others \citep{chen2025online, huang2025automated, dhillon2025scores}. 

\textbf{Hypothesis testing for AI monitoring.} Some prior works build hypothesis tests to monitor AI deployments, albeit in different contexts. \citet{vovk2021retrain} propose using conformal test martingales (CTMs) for continual monitoring of AI deployments, and \citet{prinsterwatch} develop weighted CTMs to enable test-time adaptation and analyze the cause of degradation. Safe anytime-valid testing has also been applied to track the risk of deployed models \citet{podkopaev2021tracking} and sequentially test if a classifier is fair \citep{chugg2023auditing}, among other applications \citep{i2024sequential,schirmer2026monitoring, timans2025continuous}. \citep{jang2022sequential} applies classifier two-sample tests \citep{lopez-paz2017revisiting} to detect covariate shifts. 
More broadly, several works model class-conditional densities of classifier scores and the corresponding density ratios for evaluation \citep{shanmugam2026evaluating, welinder2013lazy, dawid1979maximum, ji2020can}.

\textbf{Statistical guarantees on LLM outputs.} Another line of work tries to add formal statistical controls to LLMs. Conformal prediction methods \citep{shafer2008tutorial} quantify uncertainty in individual predictions while providing finite-sample and black-box guarantees in those uncertainties. Conformal prediction has been applied to resample LLM generations until a minimum quality requirement is satisfied \citep{quach2024conformal}. These ideas have been further applied to control the factuality of LLM outputs \citep{cherian2024large, mohri2024language,ye2024benchmarking}. More recently, \citet{wu2025thought} applies the learn-then-test framework \citep{angelopoulos2025learn} to calibrate a stopping rule for LLM reasoning traces, using white-box access to the LLM's internal logits.

\section{Experiments}

To empirically validate \textit{e-valuator}, we compare it to other baseline methods along three axes: (1) \textbf{false alarm rate control}: ``does e-valuator maintain the false alarm rate below the desired level, $\alpha$?'' (2) \textbf{power}: ``how often does e-valuator reject unsuccessful trajectories?'' and (3) \textbf{token savings}: ``among unsuccessful trajectories, how quickly does e-valuator reject, and (relatedly) how many tokens do we save via early rejection?'' These are common hypothesis testing metrics that matter in practice:  a user would like to catch as many unsuccessful trajectories as possible, and catch them early enough to save compute. Log-optimality (Section ~\ref{sec:logoptimal}) suggests that the density ratio statistic should accumulate evidence against $\mathcal{H}_A$ quickly, which should manifest empirically as both higher power (more total rejections) and shorter run lengths under $\mathcal{H}_A$ (earlier rejections).  We conducted our experiments across six datasets and tasks using three different agent-verifier combinations, each in a different setting.


\textbf{Agents and verifiers.} We conduct experiments on two tool-calling agents, Aviary \citep{narayanan2024aviary} and Octotools \citep{lu2025octotools}, and one step-by-step reasoning model, Claude Sonnet 4. For Aviary and Octotools, we use Claude Haiku 3.5 as the verifier, asking after each tool call for a text-based probability that the trajectory thus far is successful. For the reasoning model, we use a popular pretrained process-reward model \citep{wang2024math}, which provides a logit-based probability that each step in a reasoning trace is correct. 

\textbf{Datasets.} We experiment on six datasets across two domains: (1) \textbf{mathematical reasoning} (GSM8k \citep{cobbe2021training}, MATH \citep{hendrycks2measuring}, and AIME \citep{aime_1983_2024}), and (2) \textbf{question-answering} (HotpotQA \citep{yang2018hotpotqa}, MedQA \citep{jin2021disease}, and MMLU-Pro \citep{wang2024mmlu}). We present results from all datasets except GSM8k and AIME in the main section, and provide the GSM8k and AIME results in Appendix \ref{sec:more_results}. Results for each dataset are from one agent-verifier combination. A full description of the dataset, agent, and verifier combinations is available in Appendix \ref{sec:appdx:datasets}. Our verifiers span a wide range of quality (Appendix table ~\ref{tab:verifier_quality}).

\textbf{Baselines} We compare against four baselines inspired by sequential hypothesis testing and calibration.
\begin{enumerate}
    \item The \textbf{raw verifier} uses the scores from the verifier without modification.
    The verifier provides some prediction of $\text{Pr}(Y=1|H_t)$, the probability that the agent will produce a successful output, given the trajectory $H_t$ thus far.
    For a user-specified false positive rate $\alpha$, we reject a trajectory if the score $S_t$ ever drops below $\alpha$.

    \item The \textbf{calibrated verifier} uses the same verifier but recalibrates the scores $S_t$ using the calibration set $\mathcal{D}_{cal}$. 
    Specifically, we use isotonic regression to learn a $[0, 1]$-valued transformation of the score, $S' = \hat{f}(S) \in [0, 1]$, that achieves \textit{calibration}: $\text{Pr}(Y=1|S') = S'$.
    As with the raw verifier, for a user-specified false positive rate $\alpha$, we reject a trajectory if the score $S_t$ ever drops below $\alpha$.

    \item Like e-valuator, the \textbf{PAC verifier} uses the procedure from Algorithm \ref{alg:quantile}, setting $f_t(\mathbf{S}_{[1:t]}) = 1-S_t$ to account for the fact that \textit{smaller} $S_t$ from the verifier indicate ``worse" trajectories. We then find $c_\alpha$ with Algorithm \ref{alg:quantile} and use it as the threshold for rejection.

    \item \textbf{Randomized Ville} utilizes a modern variant of Ville's inequality \citep{ramdas2026randomized} to set the rejection threshold $c_\alpha$. Specifically, the threshold is $1/\alpha$ at all steps except the final step, where we threshold at $Z/\alpha$ (with $Z \sim \text{Unif}(0,1)$), allowing for more rejections at the end of a trajectory. We provide full details in Appendix~\ref{sec:randomized}. As with the standard Ville threshold, this procedure guarantees false-alarm-rate control only on exact test martingales; when the density ratios are estimated, false alarm rate control may not hold.
\end{enumerate}

 We compare these baselines to \textbf{e-valuator}, in which we choose the threshold using the procedure in Proposition ~\ref{prop:pac}. To meet a user-specified false alarm rate $\alpha$, we set $\alpha' = 0.9\cdot \alpha$ and $\delta = 0.1\cdot \alpha$ so that $\alpha' + \delta = \alpha$. We use the same $\alpha', \delta$ split for the PAC verifier.

For the results presented in the main section, we use an 80/20 split of test/calibration data: we calibrate our method (and baselines) on 20\% of the data and test on 80\%. We compare other splits of the calibration set in Appendix \ref{sec:cal_size} and find that a few hundred calibration trajectories is sufficient to achieve good false alarm rate control and power. We visualize the $M_t$ sequences for both successful and unsuccessful trajectories in Appendix \ref{sec:more_results}. Finally, we also ablate the density ratio estimator in Appendix~\ref{sec:dens_ratio_ablation}.

\begin{figure*}[t]
    \centering
    \includegraphics[width=0.9\textwidth]{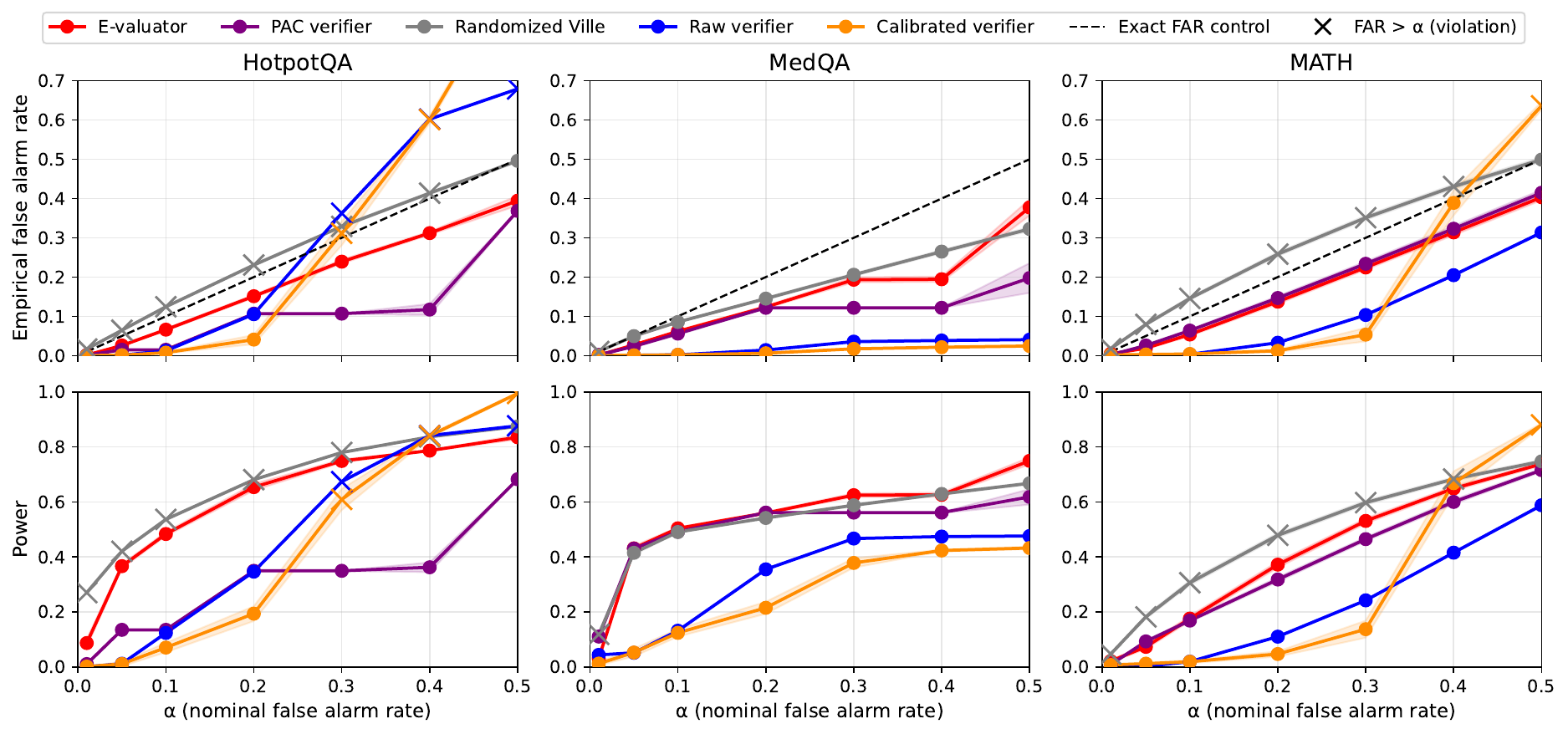}
    \caption{\textbf{E-valuator controls the false alarm rate and achieves higher power than alternative methods}. Violations of the false alarm rate control are marked with an X. Randomized Ville, the raw verifier, and the calibrated verifier occasionally provide comparable power, at the cost of inflating the false alarm rate. The PAC verifier successfully controls the false alarm rate but has consistently worse power than e-valuator across the three datasets. All plots show the 95\% CI over 50 random splits of each dataset.
    }
    \label{fig:type1}
\end{figure*}

\subsection{E-valuator provides better false alarm rate control than alternative methods}\label{sec:far}
We begin by analyzing the empirical false alarm rate achieved by \textit{e-valuator} and the competing baselines. Given a particular user-specified $\alpha$, the false alarm rate, or the rate at which we flag successful (``null") trajectories as unsuccessful (``alternative"), should be no greater than $\alpha$.

E-valuator empirically controls the false alarm rate across all choices of $\alpha$ and all datasets (Figure \ref{fig:type1}, top). The raw verifier sometimes achieves empirical false alarm rates less than the desired $\alpha$ (MedQA) but not always (HotpotQA, for $\alpha > 0.4$). The calibrated verifier, which applies isotonic regression to the raw verifier scores $S_t$ and then uses the same threshold $c_\alpha = \alpha$ on the recalibrated scores, does not achieve false alarm rates less than $\alpha$ either. Although calibration procedures such as isotonic regression aim to achieve $\text{E}(Y|S) = S$, this property does not have any direct implications on the false alarm rate. Furthermore, even if this property held at each timestep, it does not allow us to reason about the false alarm rate in the sequential hypothesis setting (see Appendix \ref{sec:theory_appdx} for further discussion).

Similarly, randomized Ville successfully controls the false alarm rate on one dataset (MedQA) but not others (HotpotQA, MATH). The false alarm rate violations are likely due to density ratio estimation error, as randomized Ville controls the false alarm rate only for exact test martingales, not necessarily approximate ones. The PAC verifier (an application of Algorithm \ref{alg:quantile} to the raw verifier scores) is indeed able to control the false alarm rate, as expected. However, it suffers in power, as we discuss below.

\subsection{E-valuator provides enhanced power over alternative methods}
Next we analyze the empirical power across the same datasets and tasks (Figure \ref{fig:type1}, bottom). That is, power is the rate at which unsuccessful trajectories (``alternative") are indeed flagged as unsuccessful. Across all datasets and all $\alpha$s, e-valuator achieves the highest power among methods that achieve empirical false alarm rates less than $\alpha$.

In those instances where the competing methods provide better power, it is at the cost of an inflated false alarm rate. For instance, in HotpotQA, at $\alpha=0.4$, the raw verifier provides a power of 0.84 but inflates the false alarm rate to 0.60. Similarly, randomized Ville has a power of 0.84 at the cost of a false alarm rate of 0.41.

In contrast,  \textit{e-valuator} provides a power of 0.78 \textit{and} controls the false alarm rate at 0.31. The PAC verifier successfully controls the false alarm rate (0.12) at the cost of substantially less power (0.36). The PAC verifier results confirm that PAC thresholding alone (without the density ratio statistic) is insufficient: it controls FAR but achieves lower power, demonstrating that log-optimality is useful for detecting unsuccessful trajectories effectively. These findings hold true across datasets. Similar plots comparing false alarm rate control and power are available in Appendix \ref{sec:more_results} for the MMLU-Pro, AIME, and GSM8k datasets.

\begin{figure*}[t]
    \centering
    \includegraphics[width=\textwidth]{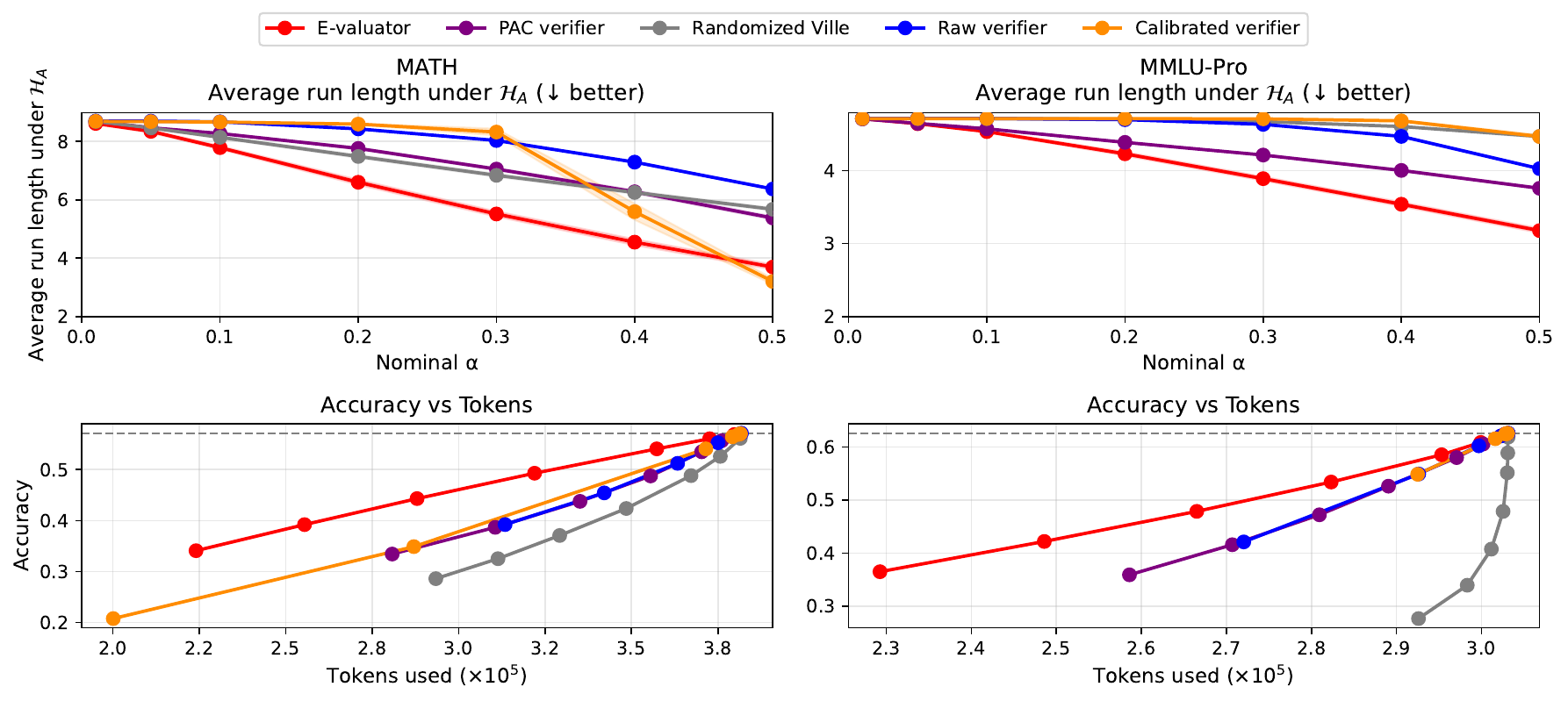}
    \caption{\textbf{E-valuator terminates unsuccessful trajectories earlier, providing better token-accuracy tradeoff.} (Top) Average run length under $\mathcal{H}_A$ ($\downarrow$ better): \textit{e-valuator} rejects unsuccessful trajectories in fewer steps (on average) than competing methods. (Bottom)  By terminating unsuccessful trajectories earlier, \textit{e-valuator} recovers a larger fraction of the original accuracy (dashed line) for a given token budget. Results shown on MATH and MMLU-Pro. False alarm rate violations are marked with an X.}
    \label{fig:tokens}
\end{figure*}

\subsection{E-valuator terminates unsuccessful trajectories earlier and recovers a larger fraction of original accuracy under a limited token budget}\label{sec:arl}
Finally, we evaluate how quickly each method detects unsuccessful trajectories. Whereas power compares the rate at which unsuccessful trajectories are rejected (at any point), in this section we compare how early each method rejects, as measured by the average run length, or average number of steps it takes to reject a trajectory, under $\mathcal{H}_A$. This is practically important: every additional step of an unsuccessful trajectory wastes tokens/compute.

\textit{E-valuator} is able to detect unsuccessful trajectories earlier (Fig \ref{fig:tokens}, row 1). For instance, at $\alpha=0.3$ on the MATH dataset, \textit{e-valuator} detects unsuccessful trajectories within 5.5 steps on average, whereas all competing baselines reject after at least 6.8 steps on average. Although the calibrated verifier achieves better average run length at $\alpha=0.5$ on MATH, it fails to control the false alarm rate (Fig. \ref{fig:type1}). Similar to power, the PAC verifier terminates unsuccessful trajectories later than \textit{e-valuator}, further confirming that the density ratio statistic (not just the PAC threshold alone) is important.

This results in a better tradeoff between token usage and accuracy (Fig \ref{fig:tokens}, row 2). We count the number of tokens that are saved (i.e., how many tokens would have otherwise been generated if the trajectory weren't terminated) and compare that to the total accuracy on the dataset. Note that using the full token budget is the ``maximal" accuracy achievable, as this would entail never terminating any trajectory.

Comparing \textit{e-valuator} to the verifier in terms of tokens saved versus total accuracy, we find that 
\textit{e-valuator} outperforms all competing baselines in recovering accuracy with fewer tokens.  For instance, on the MATH dataset, \textit{e-valuator} achieved 49\% total accuracy  (86\% of the original accuracy of 57\%) using 84\% of the original 381,804 tokens. By contrast, the raw, calibrated, and PAC verifiers each require more than 98\% of the original token count to recover comparable accuracy. Similarly, on the MMLU-Pro dataset, \textit{e-valuator} achieved 53\% of the total accuracy using just 282,316 tokens, whereas the raw, calibrated, and PAC verifiers each require more than 300,000 tokens.

\section{Discussion}
In this paper, we introduced \textit{e-valuator}, a method to improve any agent verifier model using sequential hypothesis testing. We convert the problem of detecting whether a trajectory is ``successful" or not into a hypothesis testing problem, where we distinguish between verifier scores generated from the ``successful" versus the ``unsuccessful" distribution. We introduce a PAC procedure to control the false alarm rate even with estimated density ratios.

There are promising directions for future work. First, one can relax certain assumptions to avoid estimating the \textit{full} joint density at each time $t$, such as assuming the verifier scores are i.i.d. across steps, which would allow universal inference algorithms \citep{wasserman2020universal} to construct exact e-processes from estimated density ratios. Second, one could use \textit{e-valuator} for more nuanced test-time scaling strategies than the ones explored here, such as resampling bad trajectories. Finally, \textit{e-valuator} can be used for other applications, such as early detection of unsafe or harmful trajectories.

\section{Code and Data}
Code is released on GitHub at \href{this link}{https://github.com/shuvom-s/e-valuator} and additionally as a Python package, \texttt{e-valuator}, available on PyPi at \href{this link}{https://pypi.org/project/e-valuator/}.

\section{Acknowledgments}
We thank Ian Waudby-Smith, Kexin Huang, Divya Shanmugam, Kyunghyun Cho, Manish Raghavan, the Recht Lab, Sidhika Balachandar, and Chang Ma for helpful feedback and discussions. This work was done while S.S. and D.P. were interns at Genentech. 

\bibliographystyle{unsrtnat}
\bibliography{example_paper}

@article{rabanser2026towards,
  title={Towards a science of AI agent reliability},
  author={Rabanser, Stephan and Kapoor, Sayash and Kirgis, Peter and Liu, Kangheng and Utpala, Saiteja and Narayanan, Arvind},
  journal={arXiv preprint arXiv:2602.16666},
  year={2026}
}

@inproceedings{guo2017calibration,
  title={On calibration of modern neural networks},
  author={Guo, Chuan and Pleiss, Geoff and Sun, Yu and Weinberger, Kilian Q},
  booktitle={International Conference on Machine Learning},
  pages={1321--1330},
  year={2017},
  organization={PMLR}
}

@inproceedings{huang2025automated,
  title={Automated Hypothesis Validation with Agentic Sequential Falsifications},
  author={Huang, Kexin and Jin, Ying and Li, Ryan and Li, Michael Y and Candes, Emmanuel and Leskovec, Jure},
  booktitle={International Conference on Machine Learning},
  pages={25372--25437},
  year={2025},
  organization={PMLR}
}

@article{ramdas2024hypothesis,
    author = {Ramdas, Aaditya and Wang, Ruodu},
    title = {Hypothesis Testing with E-values},
    journal = {Foundations and Trends in Statistics},
    volume = {1},
    number = {1-2},
    pages = {1-390},
    year = {2025}
}

@article{vovk2023confidence,
  title={Confidence and discoveries with e-values},
  author={Vovk, Vladimir and Wang, Ruodu},
  journal={Statistical Science},
  volume={38},
  number={2},
  pages={329--354},
  year={2023},
  publisher={Institute of Mathematical Statistics}
}

@article{vovk2021values,
  title={E-values: Calibration, combination and applications},
  author={Vovk, Vladimir and Wang, Ruodu},
  journal={The Annals of Statistics},
  volume={49},
  number={3},
  pages={1736--1754},
  year={2021},
  publisher={Institute of Mathematical Statistics}
}

@article{shin2023detectors,
  title={E-detectors: A Nonparametric Framework for Sequential Change Detection},
  author={Shin, Jaehyeok and Ramdas, Aaditya and Rinaldo, Alessandro},
  journal={The New England Journal of Statistics in Data Science},
  volume={2},
  number={2},
  pages={229--260},
  year={2023},
  publisher={New England Statistical Society}
}

@inproceedings{
podkopaev2021tracking,
title={Tracking the risk of a deployed model and detecting harmful distribution shifts},
author={Aleksandr Podkopaev and Aaditya Ramdas},
booktitle={International Conference on Learning Representations},
year={2022}
}

@article{lorden1971procedures,
  title={Procedures for reacting to a change in distribution},
  author={Lorden, Gary},
  journal={The Annals of Mathematical Statistics},
  pages={1897--1908},
  year={1971},
  publisher={JSTOR}
}

@article{wang2024math,
  title={{Math-Shepherd: Verify} and Reinforce {LLMs} Step-by-step without Human Annotations},
  author={Wang, Peiyi and Li, Lei and Shao, Zhihong and Xu, Runxin and Dai, Damai and Li, Yifei and Chen, Deli and Wu, Yu and Sui, Zhifang},
  journal={Proceedings of the 62nd Annual Meeting of the Association for Computational Linguistics (Volume 1: Long Papers)},
  pages={9426--9439},
  year={2024}
}

@inproceedings{li2025process,
  title={Process reward model with q-value rankings},
  author={Li, Wendi and Li, Yixuan},
  booktitle={International Conference on Learning Representations},
  volume={2025},
  pages={14708--14726},
  year={2025}
}

@inproceedings{
lu2025agentrewardbench,
title={{AgentRewardBench: Evaluating} Automatic Evaluations of Web Agent Trajectories},
author={Xing Han L{\`u} and Amirhossein Kazemnejad and Nicholas Meade and Arkil Patel and Dongchan Shin and Alejandra Zambrano and Karolina Stanczak and Peter Shaw and Christopher Pal and Siva Reddy},
booktitle={Second Conference on Language Modeling},
year={2025}
}

@inproceedings{prinsterwatch,
  title={{WATCH: Adaptive} Monitoring for {AI} Deployments via Weighted-Conformal Martingales},
  author={Prinster, Drew and Han, Xing and Liu, Anqi and Saria, Suchi},
  booktitle={Forty-second International Conference on Machine Learning},
year={2025}
}

@article{shanmugam2026evaluating,
  title={Evaluating multiple models using labeled and unlabeled data},
  author={Shanmugam, Divya and Sadhuka, Shuvom and Raghavan, Manish and Guttag, John and Berger, Bonnie and Pierson, Emma},
  journal={Advances in Neural Information Processing Systems},
  volume={38},
  pages={30612--30648},
  year={2026}
}

@article{ji2020can,
  title={Can {I} trust my fairness metric? {Assessing} fairness with unlabeled data and {Bayesian} inference},
  author={Ji, Disi and Smyth, Padhraic and Steyvers, Mark},
  journal={Advances in Neural Information Processing Systems},
  volume={33},
  pages={18600--18612},
  year={2020}
}

@article{narayanan2024aviary,
  title={Aviary: training language agents on challenging scientific tasks},
  author={Narayanan, Siddharth and Braza, James D and Griffiths, Ryan-Rhys and Ponnapati, Manu and Bou, Albert and Laurent, Jon and Kabeli, Ori and Wellawatte, Geemi and Cox, Sam and Rodriques, Samuel G and others},
  journal={arXiv preprint arXiv:2412.21154},
  year={2024}
}

@inproceedings{
lu2025octotools,
title={{OctoTools: An} Agentic Framework with Extensible Tools for Complex Reasoning},
author={Pan Lu and Bowen Chen and Sheng Liu and Rahul Thapa and Joseph Boen and James Zou},
booktitle={ICLR 2025 Workshop on Foundation Models in the Wild},
year={2025}
}

@article{wang2025anytime,
  title={Anytime-valid t-tests and confidence sequences for {Gaussian} means with unknown variance},
  author={Wang, Hongjian and Ramdas, Aaditya},
  journal={Sequential Analysis},
  volume={44},
  number={1},
  pages={56--110},
  year={2025},
  publisher={Taylor \& Francis}
}

@article{cobbe2021training,
  title={Training verifiers to solve math word problems},
  author={Cobbe, Karl and Kosaraju, Vineet and Bavarian, Mohammad and Chen, Mark and Jun, Heewoo and Kaiser, Lukasz and Plappert, Matthias and Tworek, Jerry and Hilton, Jacob and Nakano, Reiichiro and others},
  journal={arXiv preprint arXiv:2110.14168},
  year={2021}
}

@inproceedings{hendrycks2measuring,
  title={Measuring Mathematical Problem Solving With the MATH Dataset},
  author={Hendrycks, Dan and Burns, Collin and Kadavath, Saurav and Arora, Akul and Basart, Steven and Tang, Eric and Song, Dawn and Steinhardt, Jacob},
  booktitle={Thirty-fifth Conference on Neural Information Processing Systems Datasets and Benchmarks Track (Round 2)}
}

@inproceedings{yang2018hotpotqa,
    title = "{H}otpot{QA}: A Dataset for Diverse, Explainable Multi-hop Question Answering",
    author = "Yang, Zhilin  and
      Qi, Peng  and
      Zhang, Saizheng  and
      Bengio, Yoshua  and
      Cohen, William  and
      Salakhutdinov, Ruslan  and
      Manning, Christopher D.",
    editor = "Riloff, Ellen  and
      Chiang, David  and
      Hockenmaier, Julia  and
      Tsujii, Jun{'}ichi",
    booktitle = "Proceedings of the 2018 Conference on Empirical Methods in Natural Language Processing",
    year = "2018",
    publisher = "Association for Computational Linguistics",
    pages = "2369--2380"
}

@article{jin2021disease,
  title={What disease does this patient have? {A} large-scale open domain question answering dataset from medical exams},
  author={Jin, Di and Pan, Eileen and Oufattole, Nassim and Weng, Wei-Hung and Fang, Hanyi and Szolovits, Peter},
  journal={Applied Sciences},
  volume={11},
  number={14},
  pages={6421},
  year={2021},
  publisher={MDPI}
}

@article{wang2024mmlu,
  title={{MMLU-Pro: A} more robust and challenging multi-task language understanding benchmark},
  author={Wang, Yubo and Ma, Xueguang and Zhang, Ge and Ni, Yuansheng and Chandra, Abhranil and Guo, Shiguang and Ren, Weiming and Arulraj, Aaran and He, Xuan and Jiang, Ziyan and others},
  journal={Advances in Neural Information Processing Systems},
  volume={37},
  pages={95266--95290},
  year={2024}
}

@article{wasserman2020universal,
  title={Universal inference},
  author={Wasserman, Larry and Ramdas, Aaditya and Balakrishnan, Sivaraman},
  journal={Proceedings of the National Academy of Sciences},
  volume={117},
  number={29},
  pages={16880--16890},
  year={2020},
  publisher={National Academy of Sciences}
}

@article{wald1948optimum,
  title={Optimum character of the sequential probability ratio test},
  author={Wald, Abraham and Wolfowitz, Jacob},
  journal={The Annals of Mathematical Statistics},
  pages={326--339},
  year={1948},
  publisher={JSTOR}
}

@article{waudby2024anytime,
  title={Anytime-valid off-policy inference for contextual bandits},
  author={Waudby-Smith, Ian and Wu, Lili and Ramdas, Aaditya and Karampatziakis, Nikos and Mineiro, Paul},
  journal={ACM/IMS Journal of Data Science},
  volume={1},
  number={3},
  pages={1--42},
  year={2024},
  publisher={ACM New York, NY}
}

@article{waudby2023distribution,
  title={Distribution-uniform anytime-valid sequential inference},
  author={Waudby-Smith, Ian and Kennedy, Edward H and Ramdas, Aaditya},
  journal={arXiv preprint arXiv:2311.03343},
  year={2023}
}

@inproceedings{grunwald2020safe,
  title={Safe testing},
  author={Gr{\"u}nwald, Peter and de Heide, Rianne and Koolen, Wouter M},
  booktitle={2020 Information Theory and Applications Workshop (ITA)},
  pages={1--54},
  year={2020},
  organization={IEEE}
}

@article{uesato2022solving,
  title={Solving math word problems with process-and outcome-based feedback},
  author={Uesato, Jonathan and Kushman, Nate and Kumar, Ramana and Song, Francis and Siegel, Noah and Wang, Lisa and Creswell, Antonia and Irving, Geoffrey and Higgins, Irina},
  journal={arXiv preprint arXiv:2211.14275},
  year={2022}
}

@article{khalifa2025process,
  title={Process reward models that think},
  author={Khalifa, Muhammad and Agarwal, Rishabh and Logeswaran, Lajanugen and Kim, Jaekyeom and Peng, Hao and Lee, Moontae and Lee, Honglak and Wang, Lu},
  journal={arXiv preprint arXiv:2504.16828},
  year={2025}
}

@inproceedings{lightman2024let,
  title={Let's verify step by step},
  author={Lightman, Hunter and Kosaraju, Vineet and Burda, Yuri and Edwards, Harrison and Baker, Bowen and Lee, Teddy and Leike, Jan and Schulman, John and Sutskever, Ilya and Cobbe, Karl},
  booktitle={International Conference on Learning Representations},
  volume={2024},
  pages={39578--39601},
  year={2024}
}

@inproceedings{zheng2025processbench,
    title = "{P}rocess{B}ench: Identifying Process Errors in Mathematical Reasoning",
    author = "Zheng, Chujie  and
      Zhang, Zhenru  and
      Zhang, Beichen  and
      Lin, Runji  and
      Lu, Keming  and
      Yu, Bowen  and
      Liu, Dayiheng  and
      Zhou, Jingren  and
      Lin, Junyang",
    editor = "Che, Wanxiang  and
      Nabende, Joyce  and
      Shutova, Ekaterina  and
      Pilehvar, Mohammad Taher",
    booktitle = "Proceedings of the 63rd Annual Meeting of the Association for Computational Linguistics (Volume 1: Long Papers)",
    year = "2025",
    publisher = "Association for Computational Linguistics",
    pages = "1009--1024"
}

@inproceedings{bavaresco2024llms,
    title = "{LLM}s instead of Human Judges? {A} Large Scale Empirical Study across 20 {NLP} Evaluation Tasks",
    author = "Bavaresco, Anna  and
      Bernardi, Raffaella  and
      Bertolazzi, Leonardo  and
      Elliott, Desmond  and
      Fern{\'a}ndez, Raquel  and
      Gatt, Albert  and
      Ghaleb, Esam  and
      Giulianelli, Mario  and
      Hanna, Michael  and
      Koller, Alexander  and
      Martins, Andre  and
      Mondorf, Philipp  and
      Neplenbroek, Vera  and
      Pezzelle, Sandro  and
      Plank, Barbara  and
      Schlangen, David  and
      Suglia, Alessandro  and
      Surikuchi, Aditya K  and
      Takmaz, Ece  and
      Testoni, Alberto",
    editor = "Che, Wanxiang  and
      Nabende, Joyce  and
      Shutova, Ekaterina  and
      Pilehvar, Mohammad Taher",
    booktitle = "Proceedings of the 63rd Annual Meeting of the Association for Computational Linguistics (Volume 2: Short Papers)",
    year = "2025",
    publisher = "Association for Computational Linguistics",
    pages = "238--255"
}

@inproceedings{creswellselection,
  title={Selection-Inference: Exploiting Large Language Models for Interpretable Logical Reasoning},
  author={Creswell, Antonia and Shanahan, Murray and Higgins, Irina},
  booktitle={The Eleventh International Conference on Learning Representations}
}

@inproceedings{jang2022sequential,
  title={Sequential covariate shift detection using classifier two-sample tests},
  author={Jang, Sooyong and Park, Sangdon and Lee, Insup and Bastani, Osbert},
  booktitle={International Conference on Machine Learning},
  pages={9845--9880},
  year={2022},
  organization={PMLR}
}

@inproceedings{
lopez-paz2017revisiting,
title={Revisiting Classifier Two-Sample Tests},
author={David Lopez-Paz and Maxime Oquab},
booktitle={International Conference on Learning Representations},
year={2017}
}

@inproceedings{welinder2013lazy,
  title={A lazy man's approach to benchmarking: {Semisupervised} classifier evaluation and recalibration},
  author={Welinder, Peter and Welling, Max and Perona, Pietro},
  booktitle={Proceedings of the IEEE Conference on Computer Vision and Pattern Recognition},
  pages={3262--3269},
  year={2013}
}

@article{dawid1979maximum,
  title={Maximum likelihood estimation of observer error-rates using the {EM} algorithm},
  author={Dawid, Alexander Philip and Skene, Allan M},
  journal={Journal of the Royal Statistical Society: Series C (Applied Statistics)},
  volume={28},
  number={1},
  pages={20--28},
  year={1979},
  publisher={Wiley Online Library}
}

@inproceedings{wu2025thought,
    title = "Thought calibration: {Efficient} and confident test-time scaling",
    author = "Wu, Menghua  and
      Zhou, Cai  and
      Bates, Stephen  and
      Jaakkola, Tommi",
    editor = "Christodoulopoulos, Christos  and
      Chakraborty, Tanmoy  and
      Rose, Carolyn  and
      Peng, Violet",
    booktitle = "Proceedings of the 2025 Conference on Empirical Methods in Natural Language Processing",
    year = "2025",
    publisher = "Association for Computational Linguistics",
    pages = "14291--14305"
}

@article{angelopoulos2025learn,
  title={Learn then test: Calibrating predictive algorithms to achieve risk control},
  author={Angelopoulos, Anastasios N and Bates, Stephen and Cand{\`e}s, Emmanuel J and Jordan, Michael I and Lei, Lihua},
  journal={The Annals of Applied Statistics},
  volume={19},
  number={2},
  pages={1641--1662},
  year={2025},
  publisher={Institute of Mathematical Statistics}
}

@inproceedings{quach2024conformal,
  title={Conformal language modeling},
  author={Quach, Victor and Fisch, Adam and Schuster, Tal and Yala, Adam and Sohn, Jae Ho and Jaakkola, Tommi and Barzilay, Regina},
  booktitle={International Conference on Learning Representations},
  volume={2024},
  pages={11654--11681},
  year={2024}
}

@article{cherian2024large,
  title={Large language model validity via enhanced conformal prediction methods},
  author={Cherian, John and Gibbs, Isaac and Candes, Emmanuel},
  journal={Advances in Neural Information Processing Systems},
  volume={37},
  pages={114812--114842},
  year={2024}
}

@inproceedings{vovk2021retrain,
  title={Retrain or not retrain: Conformal test martingales for change-point detection},
  author={Vovk, Vladimir and Petej, Ivan and Nouretdinov, Ilia and Ahlberg, Ernst and Carlsson, Lars and Gammerman, Alex},
  booktitle={Conformal and Probabilistic Prediction and Applications},
  pages={191--210},
  year={2021},
  organization={PMLR}
}

@InProceedings{timans2025continuous,
  title = 	 {On Continuous Monitoring of Risk Violations under Unknown Shift},
  author =       {Timans, Alexander and Verma, Rajeev and Nalisnick, Eric and Naesseth, Christian A.},
  booktitle = 	 {Proceedings of the Forty-first Conference on Uncertainty in Artificial Intelligence},
  pages = 	 {4204--4226},
  year = 	 {2025},
  editor = 	 {Chiappa, Silvia and Magliacane, Sara},
  volume = 	 {286},
  series = 	 {Proceedings of Machine Learning Research},
  publisher =    {PMLR}
}

@article{i2024sequential,
  title={Sequential harmful shift detection without labels},
  author={I Amoukou, Salim and Bewley, Tom and Mishra, Saumitra and Lecue, Freddy and Magazzeni, Daniele and Veloso, Manuela},
  journal={Advances in Neural Information Processing Systems},
  volume={37},
  pages={129279--129302},
  year={2024}
}

@article{schirmer2026monitoring,
  title={Monitoring risks in test-time adaptation},
  author={Schirmer, Mona and Jazbec, Metod and Andersson Naesseth, Christian and Nalisnick, Eric},
  journal={Advances in Neural Information Processing Systems},
  volume={38},
  pages={80915--80948},
  year={2026}
}

@article{shafer2008tutorial,
  title={A tutorial on conformal prediction.},
  author={Shafer, Glenn and Vovk, Vladimir},
  journal={Journal of Machine Learning Research},
  volume={9},
  number={3},
  year={2008}
}

@inproceedings{mohri2024language,
  title={Language models with conformal factuality guarantees},
  author={Mohri, Christopher and Hashimoto, Tatsunori},
  booktitle={Proceedings of the 41st International Conference on Machine Learning},
  pages={36029--36047},
  year={2024}
}

@article{ye2024benchmarking,
  title={Benchmarking {LLMs} via uncertainty quantification},
  author={Ye, Fanghua and Yang, Mingming and Pang, Jianhui and Wang, Longyue and Wong, Derek and Yilmaz, Emine and Shi, Shuming and Tu, Zhaopeng},
  journal={Advances in Neural Information Processing Systems},
  volume={37},
  pages={15356--15385},
  year={2024}
}

@article{wang2025spatialagent,
  title={{SpatialAgent: An} autonomous {AI} agent for spatial biology},
  author={Wang, Hanchen and He, Yichun and Coelho, Paula P and Bucci, Matthew and Nazir, Abbas and Chen, Bob and Trinh, Linh and Zhang, Serena and Huang, Kexin and Chandrasekar, Vineethkrishna and others},
  journal={bioRxiv},
  pages={2025--04},
  year={2025},
  publisher={Cold Spring Harbor Laboratory}
}

@article{huang2025biomni,
  title={Biomni: A general-purpose biomedical {AI} agent},
  author={Huang, Kexin and Zhang, Serena and Wang, Hanchen and Qu, Yuanhao and Lu, Yingzhou and Roohani, Yusuf and Li, Ryan and Qiu, Lin and Li, Gavin and Zhang, Junze and others},
  journal={biorxiv},
  year={2025}
}

@article{swanson2025virtual,
  title={The Virtual Lab of {AI} agents designs new {SARS-CoV-2} nanobodies},
  author={Swanson, Kyle and Wu, Wesley and Bulaong, Nash L and Pak, John E and Zou, James},
  journal={Nature},
  pages={1--3},
  year={2025},
  publisher={Nature Publishing Group UK London}
}

@article{gottweis2025towards,
  title={Towards an {AI} co-scientist},
  author={Gottweis, Juraj and Weng, Wei-Hung and Daryin, Alexander and Tu, Tao and Palepu, Anil and Sirkovic, Petar and Myaskovsky, Artiom and Weissenberger, Felix and Rong, Keran and Tanno, Ryutaro and others},
  journal={arXiv preprint arXiv:2502.18864},
  year={2025}
}

@article{brown2019superhuman,
  title={Superhuman {AI} for multiplayer poker},
  author={Brown, Noam and Sandholm, Tuomas},
  journal={Science},
  volume={365},
  number={6456},
  pages={885--890},
  year={2019},
  publisher={American Association for the Advancement of Science}
}

@article{silver2017mastering,
  title={Mastering the game of go without human knowledge},
  author={Silver, David and Schrittwieser, Julian and Simonyan, Karen and Antonoglou, Ioannis and Huang, Aja and Guez, Arthur and Hubert, Thomas and Baker, Lucas and Lai, Matthew and Bolton, Adrian and others},
  journal={Nature},
  volume={550},
  number={7676},
  pages={354--359},
  year={2017},
  publisher={Nature Publishing Group UK London}
}

@article{bekey1998autonomous,
  title={On autonomous robots},
  author={Bekey, George A},
  journal={The Knowledge Engineering Review},
  volume={13},
  number={2},
  pages={143--146},
  year={1998},
  publisher={Cambridge University Press}
}

@book{doob1939jean,
  title={Etude critique de la notion de collectif},
  author={Ville, Jean},
  volume={3},
  year={1939},
  publisher={Gauthier-Villars Paris}
}

@article{chugg2023auditing,
  title={Auditing fairness by betting},
  author={Chugg, Ben and Cortes-Gomez, Santiago and Wilder, Bryan and Ramdas, Aaditya},
  journal={Advances in Neural Information Processing Systems},
  volume={36},
  pages={6070--6091},
  year={2023}
}

@inproceedings{
dhillon2025scores,
title={E-Scores for (In)Correctness Assessment of Generative Model Outputs},
author={Guneet S. Dhillon and Javier Gonzalez and Teodora Pandeva and Alicia Curth},
booktitle={The 29th International Conference on Artificial Intelligence and Statistics},
year={2026}
}

@article{ramdas2026randomized,
  title={Randomized and exchangeable improvements of {Markov’s, Chebyshev’s and Chernoff’s} inequalities},
  author={Ramdas, Aaditya and Manole, Tudor},
  journal={Statistical Science},
  volume={41},
  number={1},
  pages={121--142},
  year={2026},
  publisher={Institute of Mathematical Statistics}
}

@inproceedings{chen2025online,
  title={Online Detection of LLM-Generated Texts via Sequential Hypothesis Testing by Betting},
  author={Chen, Can and Wang, Jun-Kun},
  booktitle={International Conference on Machine Learning},
  pages={9231--9276},
  year={2025},
  organization={PMLR}
}

@article{bates2023testing,
  title={Testing for outliers with conformal p-values},
  author={Bates, Stephen and Cand{\`e}s, Emmanuel and Lei, Lihua and Romano, Yaniv and Sesia, Matteo},
  journal={The Annals of Statistics},
  volume={51},
  number={1},
  pages={149--178},
  year={2023},
  publisher={Institute of Mathematical Statistics}
}

@article{park2026know,
  title={Know What You Don't Know: Uncertainty Calibration of Process Reward Models},
  author={Park, Young-Jin and Greenewald, Kristjan and Alimohammadi, Kaveh and Wang, Hao and Azizan, Navid},
  journal={Advances in Neural Information Processing Systems},
  volume={38},
  pages={38852--38895},
  year={2026}
}

@inproceedings{you2025probabilistic,
  title={Probabilistic soundness guarantees in {LLM} reasoning chains},
  author={You, Weiqiu and Xue, Anton and Havaldar, Shreya and Rao, Delip and Jin, Helen and Callison-Burch, Chris and Wong, Eric},
  booktitle={Proceedings of the 2025 Conference on Empirical Methods in Natural Language Processing},
  pages={7517--7536},
  year={2025}
}

@inproceedings{li2025generation,
  title={From generation to judgment: Opportunities and challenges of {LLM}-as-a-judge},
  author={Li, Dawei and Jiang, Bohan and Huang, Liangjie and Beigi, Alimohammad and Zhao, Chengshuai and Tan, Zhen and Bhattacharjee, Amrita and Jiang, Yuxuan and Chen, Canyu and Wu, Tianhao and others},
  booktitle={Proceedings of the 2025 Conference on Empirical Methods in Natural Language Processing},
  pages={2757--2791},
  year={2025}
}

@article{clopper1934use,
  title={The use of confidence or fiducial limits illustrated in the case of the binomial},
  author={Clopper, Charles J and Pearson, Egon S},
  journal={Biometrika},
  volume={26},
  number={4},
  pages={404--413},
  year={1934},
  publisher={JSTOR}
}

@inproceedings{johari2017peeking,
  title={Peeking at {A/B} tests: Why it matters, and what to do about it},
  author={Johari, Ramesh and Koomen, Pete and Pekelis, Leonid and Walsh, David},
  booktitle={Proceedings of the 23rd ACM SIGKDD International Conference on Knowledge Discovery and Data Mining},
  pages={1517--1525},
  year={2017}
}

@ARTICLE{bickel2009,
  title   = "Discriminative Learning Under Covariate Shift",
  author  = "Bickel, S and Bruckner, M and Scheffer, T",
  journal = "J. Mach. Learn. Res.",
  volume  =  10,
  pages   = "2137--2155",
  year    =  2009
}

@article{gutmann2012noise,
  title={Noise-contrastive estimation of unnormalized statistical models, with applications to natural image statistics.},
  author={Gutmann, Michael U and Hyv{\"a}rinen, Aapo},
  journal={Journal of machine learning research},
  volume={13},
  number={2},
  year={2012}
}

@dataset{aime_1983_2024,
  author = {Hemish Veeraboina},
  title = {{AIME} Problem Set 1983-2024},
  year = {2023},
  publisher = {Kaggle},
  url = {https://www.kaggle.com/datasets/hemishveeraboina/aime-problem-set-1983-2024}
}

@article{leong2025steering,
  title={Steering towards safe self-driving laboratories},
  author={Leong, Shi Xuan and Griesbach, Caleb E and Zhang, Rui and Darvish, Kourosh and Zhao, Yuchi and Mandal, Abhijoy and Zou, Yunheng and Hao, Han and Bernales, Varinia and Aspuru-Guzik, Al{\'a}n},
  journal={Nature Reviews Chemistry},
  volume={9},
  number={10},
  pages={707--722},
  year={2025},
  publisher={Nature Publishing Group UK London}
}

@article{qu2026crispr,
  title={CRISPR-GPT for agentic automation of gene-editing experiments},
  author={Qu, Yuanhao and Huang, Kaixuan and Yin, Ming and Zhan, Kanghong and Liu, Dyllan and Yin, Di and Cousins, Henry C and Johnson, William A and Wang, Xiaotong and Shah, Mihir and others},
  journal={Nature Biomedical Engineering},
  volume={10},
  number={2},
  pages={245--258},
  year={2026},
  publisher={Nature Publishing Group UK London}
}

@INPROCEEDINGS{Gebreab2024-ej,
  title     = "{LLM}-based framework for administrative task automation in
               healthcare",
  author    = "Gebreab, Senay A and Salah, Khaled and Jayaraman, Raja and Habib
               ur Rehman, Muhammad and Ellaham, Samer",
  booktitle = "2024 12th International Symposium on Digital Forensics and
               Security (ISDFS)",
  publisher = "IEEE",
  pages     = "1--7",
  month     =  apr,
  year      =  2024,
  language  = "en"
}

@book{vovk2005algorithmic,
  title={Algorithmic learning in a random world},
  author={Vovk, Vladimir and Gammerman, Alexander and Shafer, Glenn},
  year={2005},
  publisher={Springer}
}

@article{dandapanthula2025offline,
  title={Offline changepoint localization using a matrix of conformal p-values},
  author={Dandapanthula, Sanjit and Ramdas, Aaditya},
  journal={arXiv preprint arXiv:2505.00292},
  year={2025}
}

\section{Appendix}

\subsection{Theory}\label{sec:theory_appdx}

\subsubsection{Proof of Proposition \ref{prop:pac}}\label{sec:proof_pac}

\newtheorem*{repproposition}{Proposition \ref{prop:pac} (PAC thresholding)}
\begin{repproposition}
    Let $\{M_t\}_{t \in \mathbb{N}_+}$ denote any scoring process where $M_t = f_t(\mathbf{S}_{[1:t]})$ for a sequence of (deterministic) functions $(f_t)_{t \in \mathbb{N}_+}$ and $M_t \in \mathbb{R}$.
    For fixed error level $\delta \in (0, 1)$ and quantile level $\alpha \in (0, 1)$, let $c_\alpha$ be the output of Algorithm~\ref{alg:quantile}.
    Then,
    \begin{align*}  \text{Pr}_{\mathcal{D}_{cal}} \left(\text{Pr}_{\nullh} \left(\exists \ t  \in [T]: M_t > c_\alpha \mid \mathcal{D}_\text{cal} \right) \leq \alpha \right) \geq 1 - \delta.
    \end{align*}
\end{repproposition}

\begin{proof}
    Suppose our calibration set contains $n$ successful trajectories that are i.i.d. according to the null.
    Note that $\text{Pr}_{\mathcal{H}_N} (\not\exists t : M_t > c_\alpha|\Dcal) = \text{Pr}_{\mathcal{H}_N} (\max_{t \in \mathbb{N}} M_t \leq c_\alpha|\Dcal)$, since the event that the entire sequence $(M_t)_{t=0}^\infty$ is below $c_\alpha$ is equivalent to the event that the maximum is below $c_\alpha$.
    Thus, it suffices to consider the maximum score over all steps.
    
    Compute the calibration set maxima $M^{(1)}, \ldots, M^{(n)}$. 
    Note that these maxima, $M^{(i)}, i = 1, \ldots, n$ are i.i.d. from some distribution with (unknown) CDF, $F$.
    Define the $(1 - \alpha)$-quantile of this distribution as
    \begin{align*}
        q_{1 - \alpha} := \inf \{x \in \mathbb{R}: F(x) \geq 1- \alpha\},
    \end{align*}
    such that $F(q_{1 - \alpha}-) \leq 1 - \alpha \leq F(q_{1 - \alpha})$, where $F(q-) := \lim_{x \rightarrow q^-} F(x)$.
    Our goal is to use the calibration data to construct a $(1 - \delta)$-confidence upper bound on $q_{1 - \alpha}$.
    That is, we will find $c_\alpha$ such that
    \begin{align}
    \label{eq:quantile-ub}
        \text{Pr}_{\Dcal}(c_\alpha < q_{1-\alpha}) \leq \delta.
    \end{align}
    Then, on the event $\{c_\alpha \geq q_{1 - \alpha}\}$, we have that 
     \begin{align}
        \text{Pr}_{\mathcal{H}_N} \left(\max_{t \in \mathbb{N}} M_t > c_\alpha \mid \Dcal \right) \leq \text{Pr}_{\mathcal{H}_N} \left (\max_{t \in \mathbb{N}} M_t > q_{1 - \alpha} \,\,\big|\,\, \Dcal \right) = 1 - F(q_{1 - \alpha})  \leq  1 - (1 - \alpha) = \alpha.
    \end{align}
    Since by construction, the event $\{c_\alpha \geq q_{1 - \alpha}\}$ occurs with probability at least $1 - \delta$, we have
    \begin{align}
        \text{Pr}_{\Dcal} \left(\text{Pr}_{\mathcal{H}_N} \left(\exists t : M_t > c_\alpha | \Dcal \right) \leq \alpha \right) = \text{Pr}_{\Dcal}\left(\text{Pr}_{\mathcal{H}_N} \left(\max_{t \in \mathbb{N}} M_t > c_\alpha \,\,\big|\,\, \Dcal \right) \leq \alpha \right) \geq 1 - \delta
    \end{align}
    as desired.
    
    We construct a $(1 - \delta)$-confidence upper bound on $q_{1 - \alpha}$ using the following argument, which follows ideas from \citet{clopper1934use}.
    Let $M_{(1)} \leq M_{(2)} \leq \cdots \leq M_{(n)}$ denote the order statistics of the calibration score maxima, where ties occupy successive ranks.
    We will set $c_\alpha$ to be one of these order statistics, as follows.
    Denote $K(x) := \# \{i : M^{(i)} < x\}$. 
    For any $x \in \mathbb{R}$,  $K(x) \sim \text{Binomial}(n, F(x-))$. Then for any $k$ and $x$, the events $\{M_{(k)} < x\}$ and $\{K(x) \geq k\}$ are equivalent, so
    \begin{align}
        \text{Pr}_{\Dcal}(M_{(k)} < x) = \text{Pr}_{\Dcal}(K(x) \geq k) = \text{Pr}(\text{Binomial}(n, F(x-)) \geq k).
    \end{align}
    In particular, for any $k$, at the quantile, $q_{1 - \alpha}$ we have
    \begin{align}
        \text{Pr}_{\Dcal}(M_{(k)} < q_{1 - \alpha}) = \text{Pr}(\text{Binomial}(n, F(q_{1 - \alpha}-)) \geq k)  \leq \text{Pr}(\text{Binomial}(n, 1 - \alpha) \geq k),
    \end{align}
    where the inequality holds because $F(q_{1 - \alpha}-) \leq 1 - \alpha$ and $\text{Pr}(\text{Binomial}(n, q) \geq k$ is nondecreasing in $q$.
    Therefore, Algorithm~\ref{alg:quantile} chooses $k^\star = \text{min} \{k: \text{Pr}(\text{Binomial}(n, 1-\alpha) \geq k) \leq \delta\}$ and sets $c_\alpha = M_{(k^\star)}$.
    This is the smallest order statistic such that Eq.~\eqref{eq:quantile-ub} holds.
    That is, by construction, $c_\alpha$ satisfies
    \begin{align}
    \label{eq:binomial}
        \text{Pr}_{\Dcal}(c_\alpha < q_{1-\alpha}) = \text{Pr}_{\Dcal}(M_{(k^\star)} < q_{1 - \alpha}) \leq \text{Pr}(\text{Binomial}(n, 1-\alpha) \geq k^\star) \leq \delta,
    \end{align}
    as desired.






\end{proof}

\subsubsection{Discussion of calibration sample size}
Algorithm \ref{alg:quantile} chooses an index $k$ by selecting $\min\{j \in [n] : \Pr[\text{Bin}(n, 1-\alpha) \geq j] \leq \delta\}$. If $n$ is not sufficiently large, this set may be empty. Since $\Pr[\text{Bin}(n, 1-\alpha) \geq j]$ is minimized at $j = n$ with value $(1-\alpha)^n$, the set is empty if and only if $(1-\alpha)^n > \delta$, or equivalently $n < \log \delta / \log(1-\alpha)$. In this regime, Algorithm \ref{alg:quantile} sets $c_\alpha = \infty$, in which case the procedure trivially controls the false alarm rate but never rejects.

Thus, given a calibration set of size $n$, the user should select $\delta, \alpha$ such that $n < \log \delta / \log(1-\alpha)$ to obtain non-zero power.

\newpage

\subsubsection{Proof of log-optimality of density ratio process}\label{sec:appdx:logopt}
\begin{proposition}[Log-optimality of density ratio statistic]
    Let $p_0$ and $p_1$ be the alternative and null densities, respectively. The density ratio process given by $M_t = \frac{\wrongdenst(\mathbf{S}_{[1:t]})}{\correctdenst(\mathbf{S}_{[1:t]})}$ provides log-optimal growth rate. That is, for any other e-process $(M_t')_{t=1}^\infty$ and stopping time $\tau$, $E_{\mathcal{H}_A}[\log M_\tau] \geq E_{\mathcal{H}_A} [\log M_\tau']$.
\end{proposition}

\begin{proof}
A full proof of this proposition appears as Theorem 7.11 in \citep{ramdas2024hypothesis}. We refer the reader to that textbook for a rigorous treatment of the proof.

First, note that $(M_t)_{t=1}^\infty$ is indeed a test martingale (and thus also an e-process). Let $\mathcal{F}_t = \sigma(\mathbf{S}_{[1:t]})$ be the natural filtration.
    We show that $(M_t)_{t \in \mathbb{N}}$ satisfies the definition of a test martingale for $\nullh: \mathbf{S} \sim P_1$.
    First, note that density ratios are always non-negative, so $M_t$ is always non-negative.
    Also, since by construction $M_0 = 1$, we have $\mathbb{E}_{\mathcal{H}_N}[M_0] = 1$.
    We now show that $(M_t)_{t \in \mathbb{N}}$ is a martingale under $\mathcal{H}_N$, that is, when $\mathbf{S} \sim P_1$.
    We have

    \begin{align*}
    \mathbb{E}_{\mathcal{H}_N}[M_t \mid \mathcal{F}_{t-1}]
    &= \mathbb{E}_{\mathcal{H}_N} \!\left[ \frac{\wrongdenst(\mathbf{S}_{[1:t]})}{\correctdenst(\mathbf{S}_{[1:t]})} \,\Big|\, \mathcal{F}_{t-1} \right] \\
    &= \frac{p_0(\mathbf{S}_{[1:t-1]})}{p_1(\mathbf{S}_{[1:t-1]})}
       \int \frac{\wrongdens(s_t \mid \mathbf{S}_{[1:t-1]})}{\correctdens(s_t \mid \mathbf{S}_{[1:t-1]})}
             \correctdens(s_t \mid \ \mathbf{S}_{[1:t-1]})\, ds_t \\
    &= M_{t-1} \int \wrongdens(s_t \mid \mathbf{S}_{[1:t-1]}) \, ds_t \\
    & = M_{t - 1}
\end{align*}
where the third equality holds because the integral of a density is $1$.
 Thus, the process $(M_t)_{t \in \mathbb{N}}$ is a test martingale for the null hypothesis.

 The main log-optimality result is an extension of the fact that in a non-sequential hypothesis test of $\correctdist$ against $\wrongdist$, where $\wrongdist \ll \correctdist$,  the likelihood ratio, $E = \wrongdens / \correctdens$, is the log-optimal e-variable: $\mathbb{E}_{\wrongdist}[\log E'] \leq \mathbb{E}_{\wrongdist}[\log E]$ for any other e-variable $E'$ for $\correctdist$.
    To see this, first note that it suffices to consider e-variables of the form $E' = dQ / d \correctdist$ for distributions $Q$ such that $Q \ll \wrongdist \ll \correctdist$.
    We have that
    \begin{align}
        \mathbb{E}_{\mathcal{H}_A}\left[\log \frac{E'}{E} \right] = \int \wrongdens(x) \log\left( \frac{q(x) / \correctdens(x)}{\wrongdens(x) / \correctdens(x)}\right) L(dx) = - \int \wrongdens(x) \log\left( \frac{\wrongdens(x)}{q(x)}\right) L(dx) \leq 0,
    \end{align}
    where $L$ is the reference measure that the densities $\correctdens, \wrongdens, q$ are defined with respect to.
    That is, $\mathbb{E}_{\mathcal{H}_A}[\log E'] \leq \mathbb{E}_{\mathcal{H}_A}[\log E]$.

    From here, one can extend this statement into the sequential setting of e-processes, which is done in Theorem 7.11 of the reference.
\end{proof}











\newpage

\subsubsection{Density ratio estimation procedure}

\begin{algorithm}[h]
\caption{Density ratio estimation using calibration data.}
\label{alg:dre}
\begin{algorithmic}[1]
   \STATE \textbf{Inputs:} calibration data $\mathcal{D}_\text{cal}$.
   \STATE \textbf{Output:} functions that estimate the density ratio for every step, $\{\hat{M}_t\}_{t \in \mathbb{N}_+}$.

   \STATE Split $\mathcal{D}_\text{cal}$ randomly, where $\mathcal{D}_\text{DRE} \cup \mathcal{D}_\text{threshold} = \mathcal{D}_\text{cal}$.
   
   \STATE $\hat{\pi}_1 \leftarrow \frac{1}{|\mathcal{D}_{\text{DRE}}|} \sum_{(\mathbf{S}, Y) \in \mathcal{D}_\text{DRE}} Y$ \COMMENT{Estimate class priors.}

   \STATE $T_\text{max} \leftarrow \max \{ T : Y \in \{0,1\}  \text{ both in } \mathcal{D}_\text{DRE} \text{ at len } T \}$

   \FOR{$t = 1, \ldots, T_\text{max}$}
       \STATE $\mathcal{D}_{\text{DRE}, t} \leftarrow \{(\mathbf{S}_{[1:t]}, Y): (\mathbf{S}, Y) \in \mathcal{D}_{\text{DRE}}: |\mathbf{S}| \geq t\}$
       \STATE Use $\mathcal{D}_{\text{DRE}, t}$ to train probabilistic classifier $\hat{g}_t$ that takes $\mathbf{S}_{[1:t]}$ as input and predicts $p(Y = 1 \mid \mathbf{S}_{[1:t]})$.
       \STATE $\hat{M}_t(\cdot) \leftarrow \frac{1 - \hat{g}_t(\cdot)}{\hat{g}_t(\cdot)} \frac{\hat{\pi}_1}{1 - \hat{\pi}_1}$
   \ENDFOR

   \STATE Set $\hat{M}_t(\cdot) \leftarrow \hat{M}_{T_\text{max}}(\cdot)$ for all $t > T_\text{max}$.
\end{algorithmic}
\end{algorithm}

Algorithm~\ref{alg:dre} specifies our density ratio estimation procedure. Note that we train a separate density ratio estimator per timestep $t$.

\newpage
\subsubsection{Baseline: Randomized Ville}\label{sec:randomized}

Prior work \cite{ramdas2026randomized} provides an alternative procedure for controlling FAR while increasing power in sequential testing using a \textit{randomized} variant of Ville's inequality, which motivates our use of it as a baseline. We refer to their paper for a detailed treatment of the procedure but provide the main result here for the sake of clarity.

\begin{proposition} (Randomized Ville's inequality)
    Let $(M_t)_{t \geq 0}$ be any non-negative supermartingale with respect to a filtration $\mathbb{F} = (\mathcal{F}_t)_{t \geq 0}$. Let $E[M_0] \leq 1$. Then, for any $\mathbb{F}$-stopping time $\tau$:
    $$\text{Pr}\left[\exists t < \tau : M_t \geq \frac{1}{\alpha} \text{ OR } M_\tau \geq \frac{Z}{\alpha}\right] \leq \alpha$$
    where $Z$ is super-uniform on $[0,1]$ and independent of $\mathbb{F}$ and $\alpha \in (0,1)$.
\end{proposition}
\begin{proof}
    See ~\citep{ramdas2026randomized} for the proof.
\end{proof}

The implication of the above result is that for a (perfectly estimated) density ratio process $M_t = \frac{p_0(\mathbf{S}_{[1:t]})}{p_1(\mathbf{S}_{[1:t]})}$ (see Eq.~\ref{eq:dens_ratio}), one can control FAR \textit{and} terminate by time $T$ (the end of the trajectory) by \textit{sampling} the threshold $Z$ at any $\mathbb{F}$-adapted stopping time $\tau$. $Z$ can be sampled from any super-uniform distribution on $[0,1]$, i.e., $Z$ can be sampled from any distribution whose CDF satisfies $P(Z \leq u) \leq u$ for all $u \in [0,1]$. Furthermore, the test \textit{ends} after $\tau$ and the user may not continue testing after this point.

We experimented with this approach in our randomized Ville baseline, setting the stopping time $\tau$ as the final step $T$ in the trajectory, and sampling $Z \sim \text{Unif}(0,1)$.

\newpage

\subsubsection{Toy example: Marginal calibration does not control false alarm rate}
A verifier score, $S$, satisfies \textit{marginal calibration} if
\begin{align}
    p(Y=1 \mid S) = S
\end{align}
almost surely.
Although a commonly used notion of probabilistic ``correctness", marginal calibration does not enable control of the false alarm rate, as it does not account for the base rate of the null and alternative (i.e., $p(Y=1)$ and $p(Y=0)$).
This is the case in the sequential hypothesis setting, where marginal calibration of $S_t$ at each step $t$ does not enable anytime control of the false alarm rate, but it is also the case in the simple non-sequential setting, as illustrated by the following example.

Suppose we have a verifier score, $S \in [0, 1]$, that only takes on two values: $S \in \{0.005, 0.5\}$, with $p(S=0.005) = 0.99$ and $p(S=0.5) = 0.01$.
The verifier score is marginally calibrated, so $p(Y=1 \mid s=0.005) = 0.005$ and $p(Y=1 \mid s=0.5) = 0.5$.
As a naive attempt to control the false alarm rate, such that $p(\texttt{reject} \mid Y=1) \le \alpha = 0.01$, we decide to \texttt{reject} whenever $S \le \alpha = 0.01$.
We now calculate the resulting false alarm rate.

First, the base rate of the null is
\begin{align*}
p(Y=1) &= p(Y=1 \mid S=0.005) \cdot p(S=0.005) + p(Y=1 \mid S=0.5) \cdot p(S=0.5)\\
&= 0.005  \cdot 0.99 + 0.5 \cdot 0.01 \\
&= 0.00995.
\end{align*}
The false alarm rate is $p(\text{reject} \mid Y=1)$, which is equivalent to $p(S=0.005 \mid Y=1)$ since we reject for $S \leq 0.01$. However,
\begin{align*}
p(\text{reject} \mid Y=1)
&= p(S = 0.005 \mid Y=1) \\
&= \frac{p(Y=1 \mid S=0.005) \, p(s=0.005)}{p(Y=1)} \\
&= \frac{0.005 \cdot 0.99}{0.00995} \\
&\approx 0.50 \gg 0.01.
\end{align*}
Thus, even with a marginally calibrated verifier score, $S$, rejecting the null when $p(Y=1 \mid S) \leq \alpha$ does not control the false alarm rate.

\newpage

\subsection{Additional results}\label{sec:more_results}
\begin{figure*}[t]
    \centering
    \includegraphics[width=\textwidth]{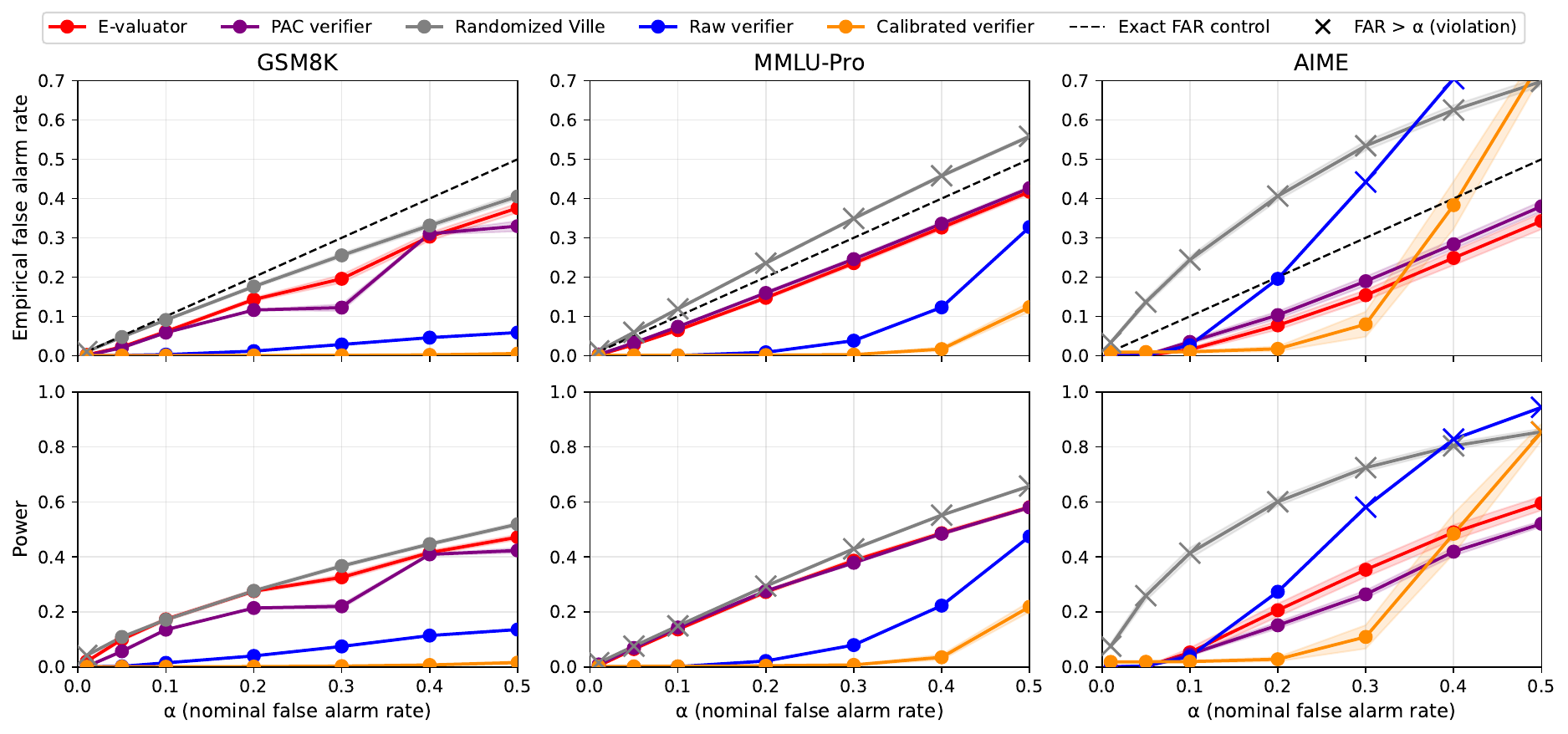}
    \caption{\textbf{GSM8k, MMLU-Pro, and AIME results}. The false alarm rate is empirically controlled for \textit{e-valuator} and the PAC verifier. Additionally, \textit{e-valuator} achieves greater (or equal) power among methods that are able to control the false alarm rate.
    }
    \label{fig:addl_far_power}
\end{figure*}
\subsubsection{Additional false alarm rate and power results}
We provide false alarm rate and power results from three  additional datasets not presented in the main text (Section ~\ref{sec:far}): GSM8k \citep{cobbe2021training},  MMLU-Pro \citep{wang2024mmlu}, and AIME \citep{aime_1983_2024} (Figure \ref{fig:addl_far_power}). \textit{E-valuator} empirically controls the false alarm rate for all choices of $\alpha$. The PAC verifier also controls the false alarm rate; however, it has worse (or equal) power on the datasets shown. Randomized Ville is able to control the false alarm rate on GSM8k but not on MMLU-Pro and AIME.

\newpage

\begin{figure*}[t]
    \centering
    \includegraphics[width=\textwidth]{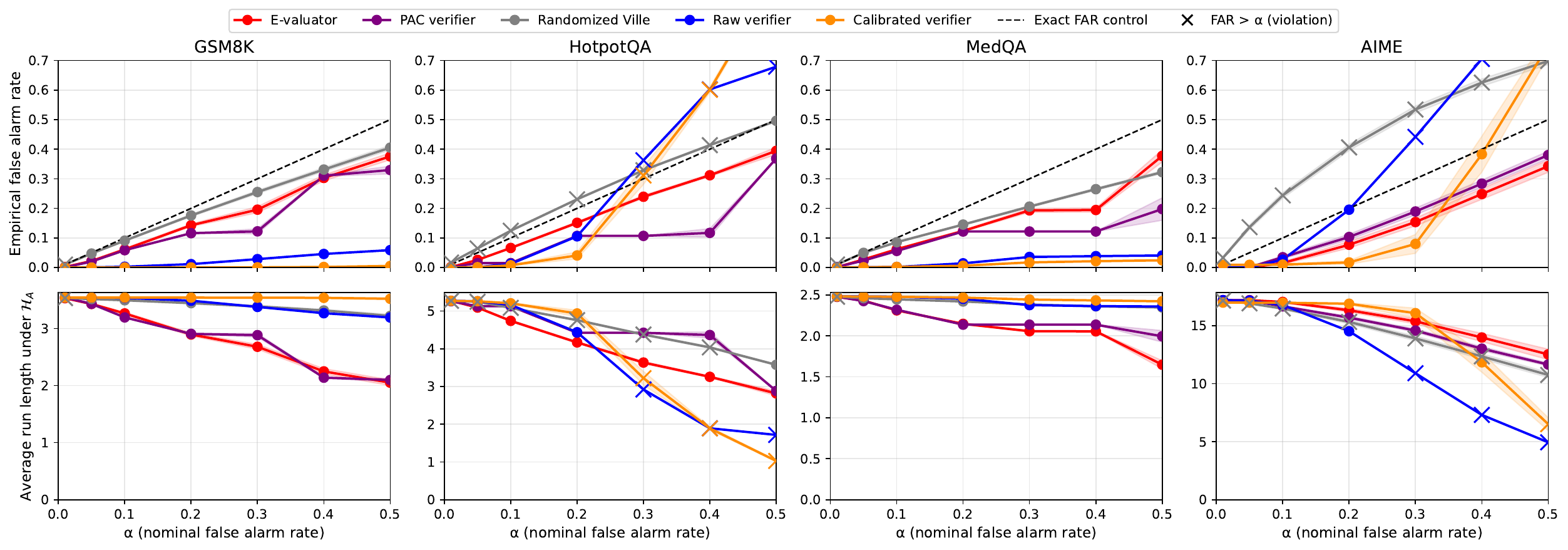}
    \caption{\textbf{Average run length results}. The average run length for unsuccessful trajectories is generally lower for e-valuator than other methods, indicating that it is able to detect unsuccessful trajectories earlier. Violations of FAR control are marked with an X.}
    \label{fig:addl_arl}
\end{figure*}

\subsubsection{Additional average run length under alternative results}
We provide average run length under alternative results from four additional datasets not presented in the main text: GSM8k \citep{cobbe2021training},  HotpotQA \citep{yang2018hotpotqa}, MedQA \citep{jin2021disease}, and AIME \citep{aime_1983_2024} (Figure~\ref{fig:addl_arl}). E-valuator provides better average run length than the PAC verifier in all but one dataset (AIME), on which it has favorable power (Figure ~\ref{fig:addl_far_power}). For instance, for HotpotQA, at $\alpha=0.4$, e-valuator detects unsuccessful trajectories within 3.2 steps on average, whereas the PAC verifier uses 4.3 steps. The raw verifier and randomized Ville baselines do not control FAR in general.

\newpage
\subsubsection{Ablations of calibration set size}\label{sec:cal_size}
\begin{figure*}[t]
    \centering
    \includegraphics[width=0.85\textwidth]{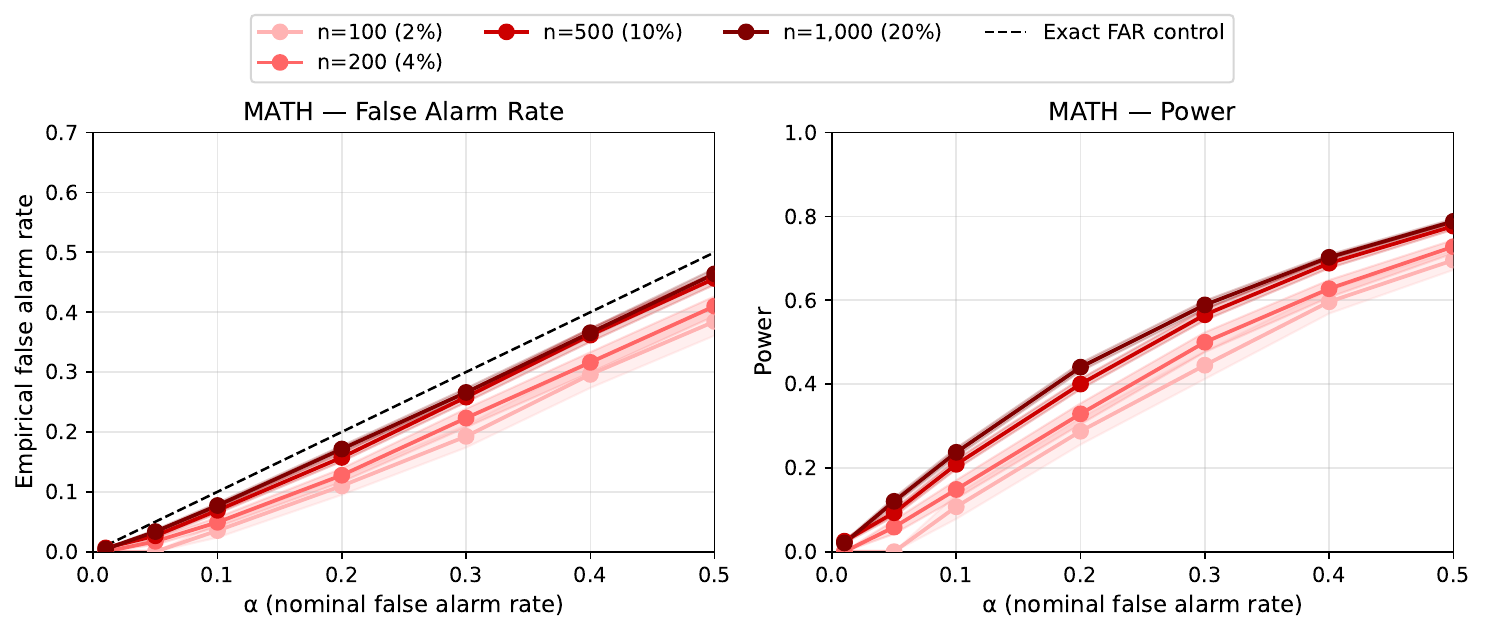}
    \caption{\textbf{Calibration set size}. On the MATH dataset, the false alarm rate is empirically controlled for all calibration set sizes. E-valuator tends to be more conservative for very small $n$ and $\alpha$, as expected.
    }
    \label{fig:cal_size}
\end{figure*}
We examine the effect of the size of the calibration set on \textit{e-valuator}. Recall that we split $\Dcal$ into $\mathcal{D}_{\text{DRE}}$ and $\mathcal{D}_{\text{threshold}}$, learning the density ratios on the former split, and estimating the rejection threshold on the latter. Because our density ratios are learned, we expect these to ratios to be more accurate as the calibration set size increases.

We run this ablation on the MATH dataset, which has 5000 total trajectories. As shown in Figure \ref{fig:cal_size}, the size of the calibration set has little effect on the empirical false alarm rates and power. However, at very small amounts of calibration data (2\%, or 100 labeled trajectories), the density ratios tend to be noisier, leading to greater variance in the false alarm rate and power.

We observe that the false alarm rates remain similar as the calibration set size increases (and all sizes control the false alarm rate).

\newpage
\subsubsection{Ablations of density ratio estimators}\label{sec:dens_ratio_ablation}
\begin{figure*}[t]
    \centering
    \includegraphics[width=\textwidth]{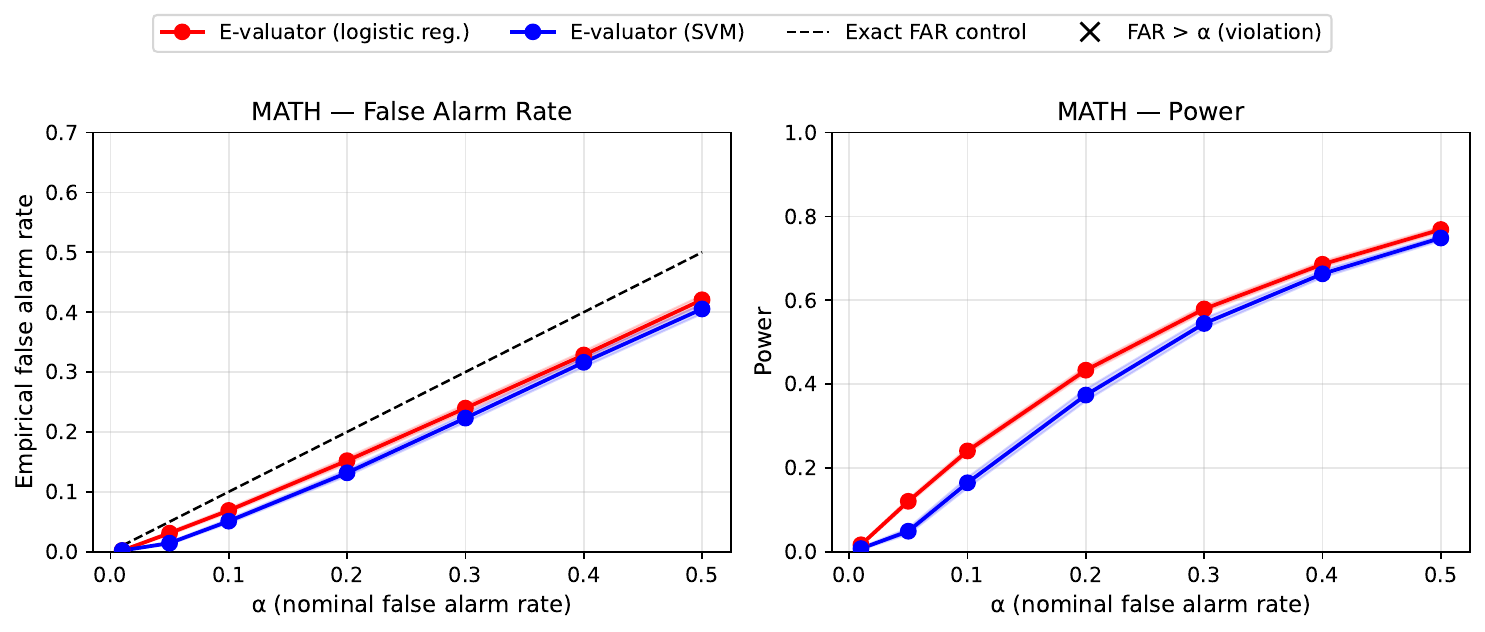}
    \caption{\textbf{Density ratio estimator}. Changing the density ratio estimator from logistic regression to SVM does not substantially change the results on the MATH dataset. False alarm rate control and power track closely across both density ratio estimation methods.
    }
    \label{fig:dre_ablation}
\end{figure*}

We examine the effect of the density ratio estimation method on \textit{e-valuator}. Here we use a support vector machine instead of logistic regression as the density ratio estimation method. On the MATH dataset, we find that the false rate and power track closely across both methods (Figure ~\ref{fig:dre_ablation}).

\newpage
\subsubsection{Example $M_t$ sequences}
\begin{figure*}[t]
    \centering
    \includegraphics[width=\textwidth]{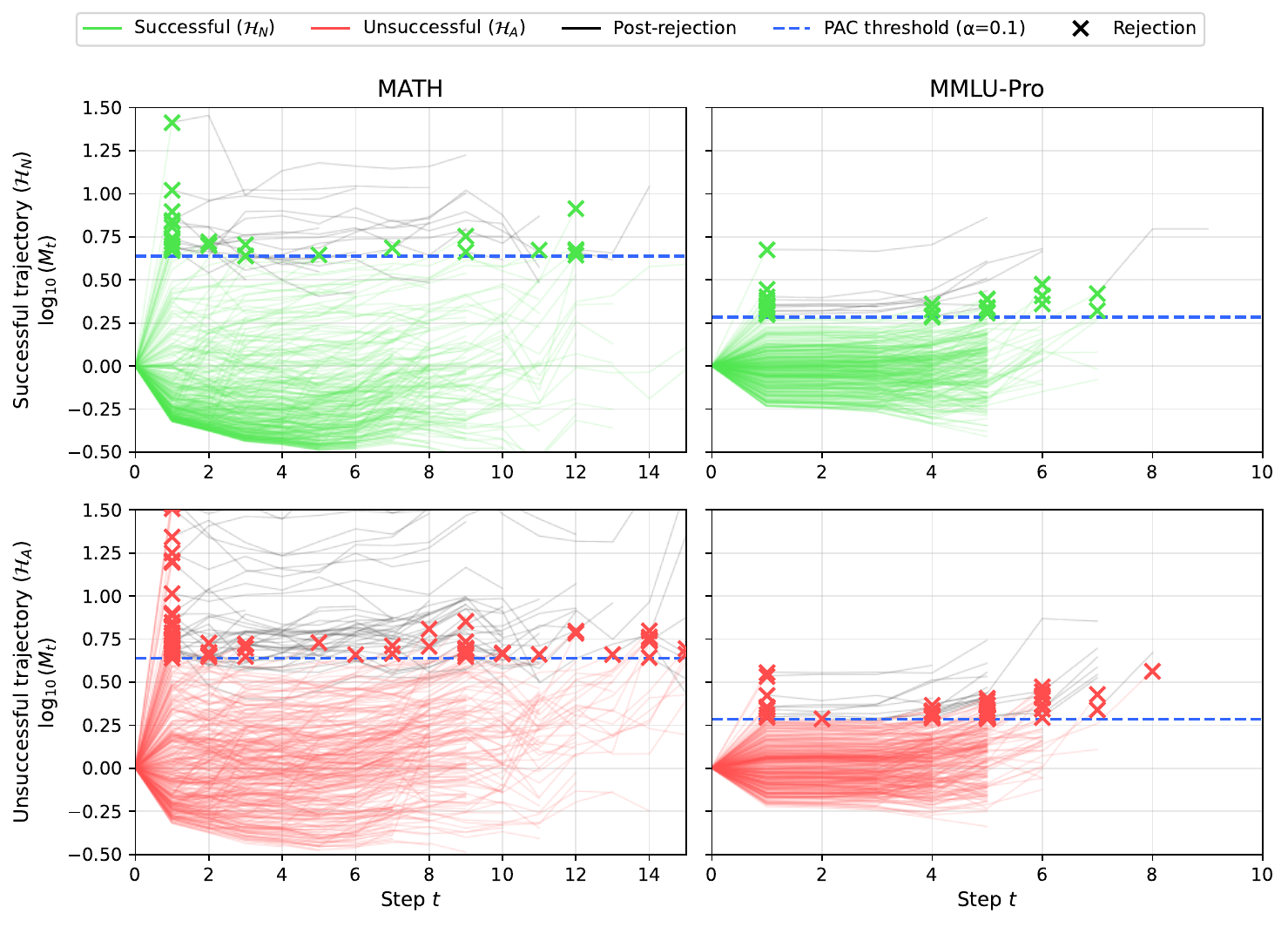}
    \caption{\textbf{Example sequences}. On the MATH dataset (left column), many of the rejections are generated at $M_1$, after the first action, indicating the first action of the agent is important in deciding success. On the MMLU-Pro dataset (right column), which has shorter trajectories, there are more rejections later in the sequence.}
    \label{fig:example_traj}
\end{figure*}

We additionally show some example $M_t$ sequences (plotted as $\log(M_t)$ for visual clarity) in Figure \ref{fig:example_traj} for the MATH and MMLU-Pro datasets. We observe that for both datasets, $\log(M_t)$ inflates above $c_\alpha$ for unsuccessful trajectories ($\mathcal{H}_A$). By contrast, few $M_t$ sequences cross the threshold $c_\alpha$ among successful trajectories. Nonetheless, there is visual heterogeneity in the sequences. In general, we expect more powerful verifiers/PRMs to provide stronger visual separation of $\mathcal{H}_N$ and $\mathcal{H}_A$ trajectories.

\newpage

\subsection{Details on datasets, agents, and verifiers}\label{sec:appdx:datasets}
We provide experiments from six different datasets. For each dataset, we use a particular agent-verifier combination, which we list in Table \ref{tab:verifiers}.

For \textbf{mathematical reasoning}, we use GSM8k \citep{cobbe2021training} for our tool-calling agent experiments and MATH \citep{hendrycks2021measuring} and AIME \citep{aime_1983_2024} for our reasoning model experiments. For \textbf{question-answering}, we use HotpotQA \citep{yang2018hotpotqa} and MedQA \citep{jin2021disease} for Aviary and OctoTools experiments, respectively,  and MMLU-Pro \citep{wang2024mmlu} for the reasoning models. We present results from all datasets except GSM8k and AIME in the main section, and provide results for these two datasets in the Appendix.

For tool-calling agents, we provide the verifier (a judge LLM) with the original problem text and list of tool calls and arguments used. We then prompt the Claude agent with the following system prompt:

\begin{quote}
You are an expert in analyzing agent trajectories and estimating the probability of success. Your final answer should be of the form: \texttt{[PROBABILITY]: [number between 0 and 1]}

The probability value should be a number between 0 and 1. LIMIT YOUR RESPONSE TO JUST \texttt{[PROBABILITY]: [number between 0 and 1]}, or else I will switch to OpenAI.

The agent in question is a LLM-based agent that uses tools to solve problems. The agent may not make more than \texttt{\{max\_tool\_calls\}} total tool calls. If it does, it will be terminated with an error.

The tools this agent can use are:
\begin{itemize}
    \item \texttt{submit\_answer}
    \item \texttt{search}
    \item \texttt{lookup}
\end{itemize}

You will be given a partial trajectory of the agent's actions. Your task is to estimate the probability of success of the agent given the partial trajectory.

Your probability should incorporate the following:
\begin{itemize}
    \item The tools that the agent has used
    \item The arguments that the agent has used, including the syntax of the arguments
    \item The problem text
    \item The number of total tool calls allowed
\end{itemize}

Here is the final answer format: \texttt{[PROBABILITY]: [number between 0 and 1]}

\texttt{\{partial\_trajectory}\}
\end{quote}

For the reasoning model, we simply provide the pretrained process reward model \citep{wang2024math} the reasoning trace and it outputs a logits-based probability that the trajectory is successful after each step. 



\begin{table*}[t]
\centering
\footnotesize
\setlength{\tabcolsep}{6pt}
\renewcommand{\arraystretch}{1.2}

\begin{tabular}{l p{2.2cm} l p{2cm} p{3.5cm}}
\hline
\textbf{Dataset} & \textbf{Domain} & \textbf{Agent} & \textbf{Verifier} & \textbf{Agent Description} \\
\hline

GSM8k \citep{cobbe2021training} 
    & Math reasoning
    & Aviary \citep{narayanan2024aviary}
    & Claude Haiku 3.5
    & Tool-calling agent for math QA. Text-based verifier model. \\

MATH \citep{hendrycks2021measuring} 
    & Math reasoning
    & Claude Sonnet 4 
    & Pretrained PRM \citep{wang2024math}
    & Multi-step reasoning model, with pretrained verifier model. \\

AIME \citep{aime_1983_2024} 
    & Math reasoning
    & Claude Sonnet 4 
    & Pretrained PRM \citep{wang2024math}
    & Multi-step reasoning agent, with pretrained verifier model. \\

HotpotQA \citep{yang2018hotpotqa} 
    & QA
    & Aviary 
    & Claude Haiku 3.5
    & Tool-calling agent for general QA. Text-based verifier model. \\

MedQA \citep{jin2021disease} 
    & QA
    & OctoTools \citep{lu2025octotools}
    & Claude Haiku 3.5
    & Tool-calling agent for medical QA. Text-based verifier model. \\

MMLU-Pro \citep{wang2024mmlu} 
    & QA
    & Claude Sonnet 4 
    & Pretrained PRM \citep{wang2024math}
    & Multi-step reasoning agent, with pretrained verifier model. \\

\hline
\end{tabular}
\caption{List of datasets, agents, and verifiers used in our experiments.}
\end{table*}\label{tab:verifiers}

\begin{table}[t]
\centering
\small
\begin{tabular}{llcc}
\hline
\textbf{Dataset} & \textbf{Verifier} & \textbf{Verifier AUC} & \textbf{Verifier ECE} \\
\hline
MedQA & Claude Haiku 3.5 & 0.700 & 0.021 \\
MATH & Pretrained PRM & 0.692 & 0.184 \\
HotpotQA & Claude Haiku 3.5 & 0.654 & 0.137 \\
AIME & Pretrained PRM & 0.624 & 0.083 \\
GSM8k & Claude Haiku 3.5 & 0.602 & 0.053 \\
MMLU-Pro & Pretrained PRM & 0.592 & 0.086 \\
\hline
\end{tabular}
\caption{Verifier quality across datasets, measured by AUC and expected calibration error (ECE), both averaged across all timesteps and trajectories.}
\label{tab:verifier_quality}
\end{table}

Table ~\ref{tab:verifiers} lists all the datasets, agents, and verifiers used. Table ~\ref{tab:verifier_quality} additional lists some statistics on the quality of these verifiers. As shown, there is substantial variation in the verifier quality. For instance, on MedQA, the verifier has strong discriminative power (AUC 0.7) and is well-calibrated (ECE 0.021). By contrast, on MATH, the verifier has strong predictive power (AUC 0.692) but is poorly calibrated (ECE 0.184). Finally, on GSM8k, the verifier has weak predictive power (AUC 0.602) but is reasonably calibrated (ECE 0.053).

\newpage

\subsubsection{Additional computational details}
We use the default hyperparameter settings in \texttt{scikit-learn} with logistic regression for all experiments presented in this paper (except the density ratio ablation experiments, for which we also use the default SVM settings in \texttt{scikit-learn}). We split $\mathcal{D}_{cal}$ 50/50 into the density ratio estimation and threshold calibration split. Given a set of verifier scores, all experiments in this paper can be completed in under a minute on a standard laptop.

\end{document}